\documentclass[letterpaper,journal]{IEEEtran}
\usepackage{array}
\usepackage{textcomp}
\usepackage{stfloats}
\usepackage{url}
\usepackage{verbatim}
\usepackage{graphicx}
\usepackage{epsfig} 
\usepackage{amsmath} 
\usepackage{amssymb}  
\usepackage{fixmath}
\usepackage[noadjust]{cite}
\usepackage[colorlinks=true, urlcolor=blue]{hyperref}

\usepackage{booktabs} 
\usepackage{multirow}

\usepackage{amsthm}
\newtheorem{theorem}{Theorem}
\newtheorem{definition}{Definition}
\newtheorem{proposition}{Proposition}
\newtheorem{corollary}{Corollary}

\newtheorem{remark}{Remark}

\usepackage{enumitem}

\usepackage[linesnumbered, ruled, vlined]{algorithm2e}
\usepackage{xcolor}

\SetCommentSty{mycommfont}

\usepackage[font=footnotesize]{subcaption}
\usepackage[font=footnotesize]{caption}

\usepackage[capitalize]{cleveref}
\crefformat{equation}{(#2#1#3)}
\Crefformat{equation}{Equation~(#2#1#3)}
\Crefname{equation}{Equation}{Eqs.}

\begin{document}

\title{
\LARGE \bf
STITCHER: Constrained Trajectory Planning in Complex Environments with Real-Time Motion Primitive Search 
}

\author{Helene J. Levy and Brett T. Lopez
\thanks{Authors are with the VECTR Laboratory, University of California, Los Angeles, Los Angeles, CA, USA. {\tt\small \{hjlevy, btlopez\}@ucla.edu}}%
}

\markboth{Journal of \LaTeX\ Class Files,~Vol.~0, No.~0, September~2025}%
{Shell \MakeLowercase{\textit{et al.}}: A Sample Article Using IEEEtran.cls for IEEE Journals}


\maketitle
\thispagestyle{empty}
\pagestyle{empty}
\setlength{\parskip}{0pt}

\begin{abstract}
Autonomous high-speed navigation through large, complex environments requires real-time generation of agile trajectories that are dynamically feasible, collision-free, and satisfy state or actuator constraints.
Modern trajectory planning techniques primarily use numerical optimization, as they enable the systematic computation of high-quality, expressive trajectories that satisfy various constraints.
However, stringent requirements on computation time and the risk of numerical instability can limit the use of optimization-based planners in safety-critical scenarios. 
This work presents an optimization-free planning framework called STITCHER that stitches short trajectory segments together with graph search to compute long-range, expressive, and near-optimal trajectories in real-time.
STITCHER outperforms modern optimization-based planners through our innovative planning architecture and several algorithmic developments that make real-time planning possible. 
Extensive simulation testing is performed to analyze the algorithmic components that make up STITCHER, along with a thorough comparison with three state-of-the-art optimization planners.
Simulation tests show that safe trajectories can be created within a few milliseconds for paths that span the entirety of two 50 m x 50 m environments. 
Hardware tests with a custom quadrotor verify that STITCHER can produce trackable paths in real-time while respecting nonconvex constraints, such as limits on tilt angle and motor forces, with flight speeds up to 63 km/h.
\end{abstract}


\vspace{0.05in}
\noindent \textbf{\small Code:} \href{https://github.com/vectr-ucla/stitcher}{\small https://github.com/vectr-ucla/stitcher}

\section{INTRODUCTION}
\IEEEPARstart{P}{lanning} collision-free, dynamically feasible trajectories in real-time through complex environments is crucial for many autonomous systems. 
As a result, trajectory planning has garnered significant interest from the research community, but meeting the reliability requirements for safety-critical real-world applications remains challenging. 
Specifically, few methods have guarantees regarding trajectory optimality and time/memory complexity without sacrificing trajectory length, computation time, or expressiveness. 
Our approach addresses this gap by combining optimal control with graph search to generate near-optimal trajectories over long distances in real-time, resulting in a framework that provides strong guarantees on path quality and algorithm complexity.

Optimization-based trajectory planning has emerged as the primary framework for autonomous systems that must navigate complex environments.
This is because constraints and performance objectives can be naturally stated in the optimization problem.
Most approaches can be broadly classified by their use of continuous or integer variables. 
Continuous variable methods employ gradient-based optimization to jointly optimize over the coefficients of basis functions (e.g., polynomials) and waypoint arrival times while imposing obstacle and state constraints \cite{Mellinger11:Minimum, Richter16:Polynomial,Oleynikova16:Continuous-time, Zhou2021:RAPTOR, Wang22:Geometrically, Ren22:Bubble, Kondo26:Mighty}.
Integer variable methods require that the free space of the environment be represented as the union of convex sets (continuous variable methods have also used this representation, e.g., \cite{Wang22:Geometrically,Ren22:Bubble, Kondo26:Mighty}) and solve a mixed-integer program for a collision-free trajectory \cite{Deits15:Computing, Deits15:Efficient, Tordesillas22:FASTER, Marcucci23:Motion}.
Despite continued innovations, these methods lack \textit{a priori} time complexity bounds and often scale very poorly with trajectory length; this is especially true for integer programming approaches.

\begin{figure}[t!]
  \centering
  \includegraphics[width=\columnwidth]{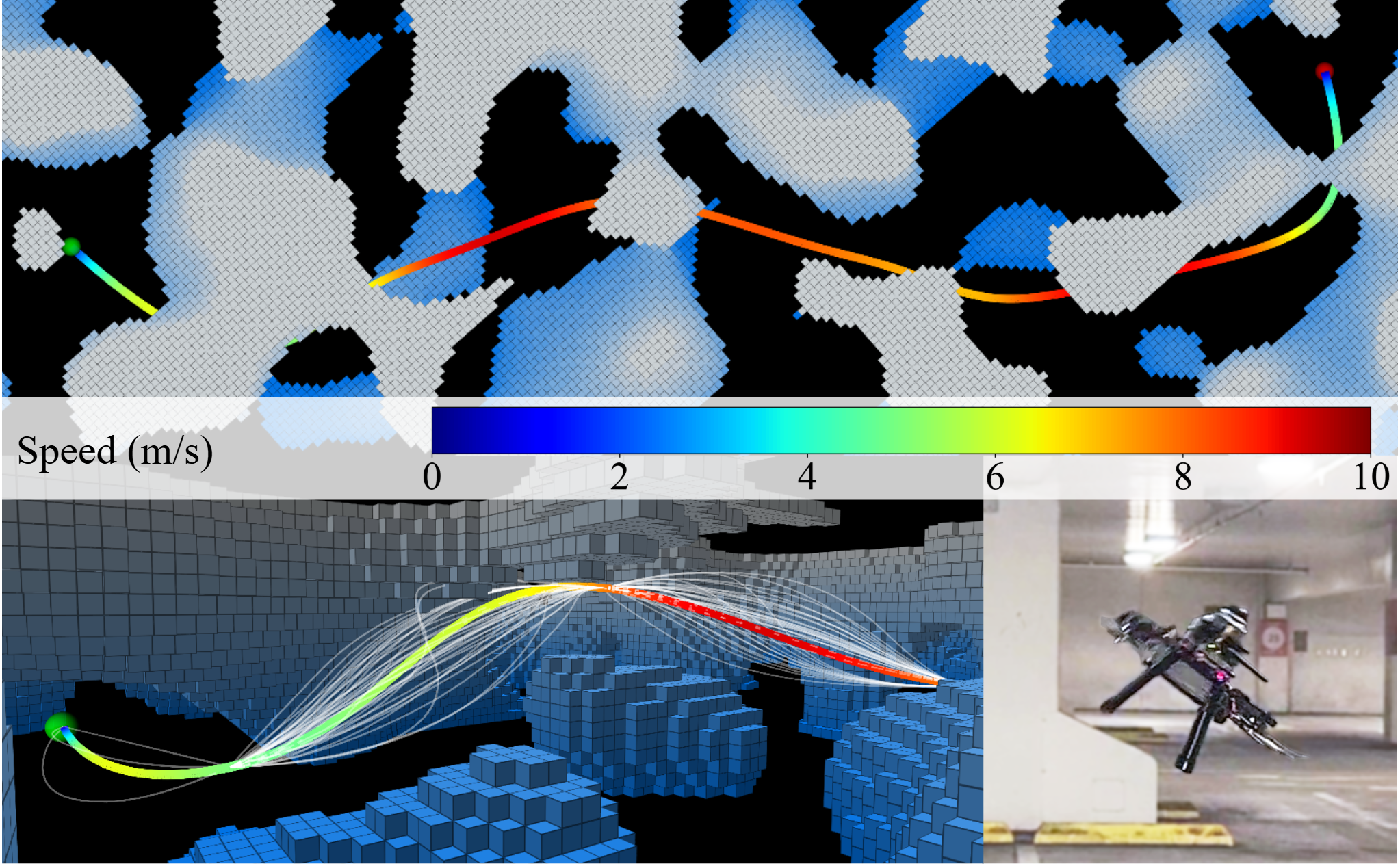}
  \caption{A trajectory (colored based on speed) generated by our proposed algorithm called STITCHER through a perlin noise environment. STITCHER searches over candidates motion primitives (white) to find a safe trajectory in real-time with time and memory complexity guarantees. Flight experiments were performed to verify dynamic feasibility of high-speed trajectory plans.}
  \label{fig:first_pic}
\vskip -0.05in    
\end{figure}

A computationally efficient alternative to optimization-based trajectory planning is the use of so-called motion primitives: a library of short length or duration trajectories that can be efficiently computed and evaluated \cite{Mueller15:A_Computationally, Florence16:Integrated, Lopez17:Aggressive, Ryll19:Efficient}.
To effectively use motion primitives, a planner must operate in a receding horizon fashion, i.e., continuously replan, because motion primitives are inherently near-sighted with their short length or duration.  
This can introduce several performance issues, e.g., myopic behavior or unexpressive trajectories, that are exacerbated in large, complex environments. 
Subsequent work has attempted to pose the problem as a graph search with nodes and edges being desired states (position, velocity, etc.) and motion primitives, respectively \cite{Liu17:Search_Based, Liu18:Search_Based, Zhou19:Robust, Foehn21:Alphapilot}.
While this allows for the creation of long-range trajectories, search times can be extremely high (seconds) because of the graph size. 
An admissible search heuristic can be used to reduce the number of node expansions required to find a solution, also known as search effort, while preserving optimality of the graph \cite{Pearl84:Heuristics}. However, designing such a search heuristic is non-trivial. 

We propose a new trajectory planning algorithm called STITCHER that can perform real-time motion primitive searches across long distances in complex environments.
STITCHER utilizes an innovative three-stage planning architecture to generate smooth and expressive trajectories by \emph{stitching} motion primitives together through graph search.
The first two stages are designed to expedite the motion primitive search in the final stage by constructing a compact but expressive search graph and search heuristic.
Specifically, given a set of waypoints computed in the first stage, we create a velocity graph by representing nodes as sampled velocities at each waypoint, and edges as quick-to-generate minimum-time trajectories.
We employ dynamic programming on this graph, using the Bellman equation to compute the cost-to-go for each node. 
Critically, the cost-to-go is then used as an admissible heuristic to efficiently guide the motion primitive search in the third stage.
We also leverage a greedy graph pre-processing step to form a compact motion primitive graph. 
We prove all graphs are finite and that the proposed heuristic is admissible.
These technical results are critical as they guarantee i) \emph{a priori} time and memory complexity bounds and ii) trajectory optimality with respect to the graph discretization. 
To further reduce computation time, we improve the collision checking procedure from \cite{Lopez17:Aggressive}, by leveraging the known free space from previous nearest-neighbor queries, bypassing the rigidity and computational complexity of free space decomposition. 
Additionally, we show that employing a simple sampling procedure in the final search stage is effective at pruning candidate trajectories that violate complex state or actuator constraints. 
STITCHER was extensively tested in two simulation environments to evaluate algorithmic innovations and assess its performance against three state-of-the-art optimization-based planners \cite{Tordesillas22:FASTER,Wang22:Geometrically,Kondo26:Mighty}.
STITCHER is shown to generate high-quality, dynamically feasible trajectories over long distances (over 50 m) with computation times in the milliseconds, and consistently generates trajectories faster than the state-of-the-art with comparable execution times.
Finally, we tested the trajectories designed by STITCHER in hardware to show that the trajectories were dynamically feasible, adhered to physical constraints, and could be tracked via standard geometric cascaded control at high speeds (up to 63 km/h).

\section{RELATED WORKS}
\subsection{Optimization-based Planning}
Designing high quality trajectories using online optimization has become a popular planning strategy as a performance index and constraints can be systematically incorporated into an optimization problem.
Optimization-based trajectory planners can be categorized using several criteria, but the clearest delineation is whether the method uses continuous or integer variables.
For methods that use only continuous variables, the work by \cite{Richter16:Polynomial} reformulated the quadratic program in \cite{Mellinger11:Minimum} to jointly optimize over polynomial endpoint derivatives and arrival times for a trajectory passing through waypoints. 
Collisions were handled by adding intermediate waypoints and redoing the trajectory optimization if the original trajectory was in collision. 
Oleynikova et al.~\cite{Oleynikova16:Continuous-time} represented obstacles using an Euclidean Signed Distance Field (ESDF) which was incorporated into a nonconvex solver as a soft constraint.
Zhou et al.~\cite{Zhou2021:RAPTOR} used a similar penalty-based method but introduced a topological path search to escape local minima. 
An alternative approach is to decompose the occupied space or free space into convex polyhedra \cite{Mellinger12:Mixed-integer, Deits15:Computing, Liu17:Planning} which can be easily incorporated as constraints in an optimization. 
The methods in \cite{Wang22:Geometrically, Ren22:Bubble, Kondo26:Mighty} treat these constraints as soft while efficiently optimizing over polynomial trajectory segments that must pass near waypoints. 
One can also use the free-space polyhedra to formulate a mixed-integer program \cite{Deits15:Efficient, Landry16:Aggressive, Tordesillas22:FASTER, Marcucci23:Motion} to bypass the nonconvexity introduced by having unknown waypoint arrival times, but at the expense of poor scalability with trajectory length and number of polyhedra. Marcucci et al.
\cite{Marcucci24:Fast} addresses scalability concerns of \cite{Marcucci23:Motion} by solving a sequence of convex problems instead of one large-scale optimization but requires an offline process for generating collision-free convex sets and does not reason about dynamics. 

\subsection{Motion Primitives}
Motion primitive planners have been proposed as an alternative to optimization-based planners to address computational complexity and numerical instability concerns.
The underlying idea of motion primitive planners is to select trajectories online from a precomputed, finite library of trajectories.
Initial work on motion primitives for quadrotors leveraged differential flatness and known solutions to specific optimal control problems to efficiently compute point-to-point trajectories in real-time \cite{Hehn11:Quadrocopter,Mueller15:A_Computationally}.
Later work employed motion primitives for receding horizon collision avoidance where primitives were efficiently generated online by sampling desired final states, and selected at each planning step based on safety and trajectory cost \cite{Howard08:State, Florence16:Integrated,Lopez17:Aggressive, Lopez17:AggressiveFOV, Ryll19:Efficient, Dharmadhikari20:Motion, Hou25:Primitive-Swarm}.
Howard et al. \cite{Howard08:State} first introduced this idea of searching over feasible trajectories of a car with a model predictive control framework. 
Subsequent works extended this methodology to quadcopters using depth images \cite{Florence16:Integrated, Ryll19:Efficient}, point clouds \cite{Lopez17:Aggressive,Lopez17:AggressiveFOV}, or ESDFs \cite{Dharmadhikari20:Motion} for motion primitive evaluations and collision avoidance.
While computationally efficient, the behavior of these planners can be myopic, leading to suboptimal behavior in complex environments which limit their use for planning long-term trajectories.

\subsection{Motion Primitive Search}
One way to address nearsightedness is to perform a graph search over motion primitives, i.e., stitch motion primitives together. 
This can be achieved by extending traditional graph search algorithms \cite{Dijkstra59:Dijkstra, Hart68:Astar, Harabor11:JPS}, which typically use coarse discrete action sets, to using a lattice of motion primitives \cite{Dolgov10:Path,Pivtoraiko11:Kinodynamic,Liu17:Search_Based, Liu18:Search_Based, Zhou19:Robust, Andersson18:Receding}. 
Graph search algorithms are an attractive method for planning due to inherent guarantees of completeness, optimality\footnote{In the context of graph search, optimality refers to resolution optimality, i.e., optimality with respect to the discretized state space.}, and bounded time and memory complexity \cite{Russell16:Artificial}.
The works by Liu et al. 
\cite{Liu17:Search_Based,Liu18:Search_Based} were some of the first works to successfully showcase a search-based algorithm using a lattice of motion primitives for use on quadcopters. 
However, these methods can be computationally expensive as generating high-quality trajectories relies on generating a large number of motion primitives for sufficient density. 
Jarin et al. \cite{Jarin21:Dispersion} addresses computation concerns by improving upon the sampling of different motion primitives, inspired by a minimum dispersion sampling method \cite{Palmieri20:Dispertio}. 
Another way to narrow the search space is by utilizing a geometric path as a prior and constraining motion primitives to pass through waypoints from the path. 
Recently, \cite{Foehn21:Alphapilot, Penicka22:Minimum, Romero22:Model, Romero22:Time} proposed an efficient motion primitive search in velocity space using minimum-time input-constrained trajectories from a double integrator model restricted to pass through a set of waypoints. 
The search can be done in real-time but the resulting bang-bang acceleration profile is dynamically infeasible for quadrotors, leading to poor tracking performance. 
An additional smoothing step, e.g., model predictive contouring control, is required to achieve sufficient trajectory smoothness \cite{Romero22:Model, Romero22:Time, Krinner24:MPCC++}. 

\subsection{Search Heuristics}
Fast graph search speed while retaining optimality guarantees can be achieved by employing an admissible search heuristic \cite{Russell16:Artificial} to guide the search to the goal state.
Constructing an informative and admissible heuristic, however, is non-trivial. 
Much of the previous work in motion primitive search overlooks the importance of the heuristic by generating a weak approximation to the goal \cite{Zhou19:Robust}, using an inadmissible heuristic which forfeits optimality guarantees \cite{Liu18:Search_Based}, or proceeding without a heuristic \cite{Foehn21:Alphapilot}. 
As a result, motion primitive search algorithms to date scale poorly in large environments and for large planning horizons, making them unsuitable for systems with limited onboard computational resources.
Paden et al. \cite{Paden17:Verification} proposed a method to systematically construct admissible heuristics for use in kinodynamic planning using sum-of-squares (SOS) programming.
However, the resulting size of the SOS program requires heuristic calculations to be performed offline.
Other strategies involve learning a search heuristic or cost-to-go \cite{Kim20:LHA*, Thayer11:Learning, Bhardwaj17:Learning, Pandy22:Learning}. Kim et al. \cite{Kim20:LHA*} uses a neural network to approximate graph node distances and provides a sub-optimality bound on the solution. 
Reinforcement and imitation learning have also been proposed for learning search heuristics \cite{Thayer11:Learning, Bhardwaj17:Learning, Pandy22:Learning}, but these works focus on minimizing node expansions rather than ensuring admissibility, sacrificing the optimality guarantees of graph search. 

\section{PROBLEM FORMULATION} 
This work is concerned with solving the following trajectory planning problem: 
\begin{equation}
\label{eq:tpp}
    \begin{aligned}
    \min_{\mathbold{u} \, \in \, \mathcal{U}} \quad &J = r(T) +  \int_{0}^{T} q(\mathbf{x},\mathbold{u}) \,  dt \\
    \text{s.t.} \quad &\dot{\mathbf{x}} = A\mathbf{x}+B\mathbold{u} \\
    & \mathbf{x} \in \mathcal{X}_s, \ \mathbf{x} \notin \mathcal{X}_{obst}, \ \mathbold{u} \in \mathcal{U} \\
    &\mathbf{x}(0) = \mathbf{x}_0, \  \mathbf{x}(T) = \mathbf{x}_f,
    \end{aligned}
\end{equation}
where $\mathbf{x}\in \mathbb{R}^n$ is the state that must satisfy state $\mathcal{X}_s$ and obstacle (collision) $\mathcal{X}_{obst}$ constraints, $\mathbold{u} \in \mathbb{R}^m$ is the control input that must satisfy actuator constraints $\mathcal{U}$, $A \in \mathbb{R}^{n\times n}$ and $B\in \mathbb{R}^{n\times m}$ govern the system's dynamics and are assumed to take a multi-axis chain of integrators form, and $r : \mathbb{R}_+ \rightarrow \mathbb{R}_+$ and $q:\mathbb{R}^n \times \mathbb{R}^m \rightarrow \mathbb{R}_+$ are the terminal and stage cost, respectively. 
The goal is to find an optimal final time $T^*$ and feasible optimal state trajectory $\mathbf{x}^*(t)$ with a corresponding control input sequence $\mathbold{u}^*(t)$ for $t \in [0\ T^*]$ that steers the system from an initial state $\mathbf{x}_0$ to a desired final state $\mathbf{x}_f$ that minimizes the cost functional $J$. 
We assume that the desired final state $\mathbf{x}_f$ is reachable from the initial state $\mathbf{x}_0$ under actuator constraints $\mathcal{U}$.
While the dynamics are linear in \cref{eq:tpp}, many nonlinear systems can be placed into the linear control affine form if they are differentially flat, e.g., vertical take-off and landing (VTOL) vehicles like quadrotors, capturing a large class of systems of interest.
In many cases, the state vector can be $\mathbf{x} = ( \mathbold{r}, \,  {\mathbold{v}}, \, {\mathbold{a}}, \, \dots, \, \mathbold{r}^{(p-1)} )^\top$ and the control input can be $\mathbold{u} = \mathbold{r}^{(p)}$ where $\mathbold{r} = (x,\,y,\,z)^\top$ is the position of the vehicle in some reference frame.

\subsection{Background: Motion Primitives}
We define motion primitives to be closed-form solutions to specific optimal control problems. 
In this work, we will restrict our attention to the following two optimal control problems: the input-constrained minimum-time problem for a double integrator and the linear quadratic minimum time problem for a $p$-th order integrator.
We will briefly review each optimal control problem and the structure of its solution. 
The formulations will be presented for a single axis, but can be repeated for all three position coordinates.

\begin{figure*}[t]
  \centering
  \includegraphics[width=\textwidth]{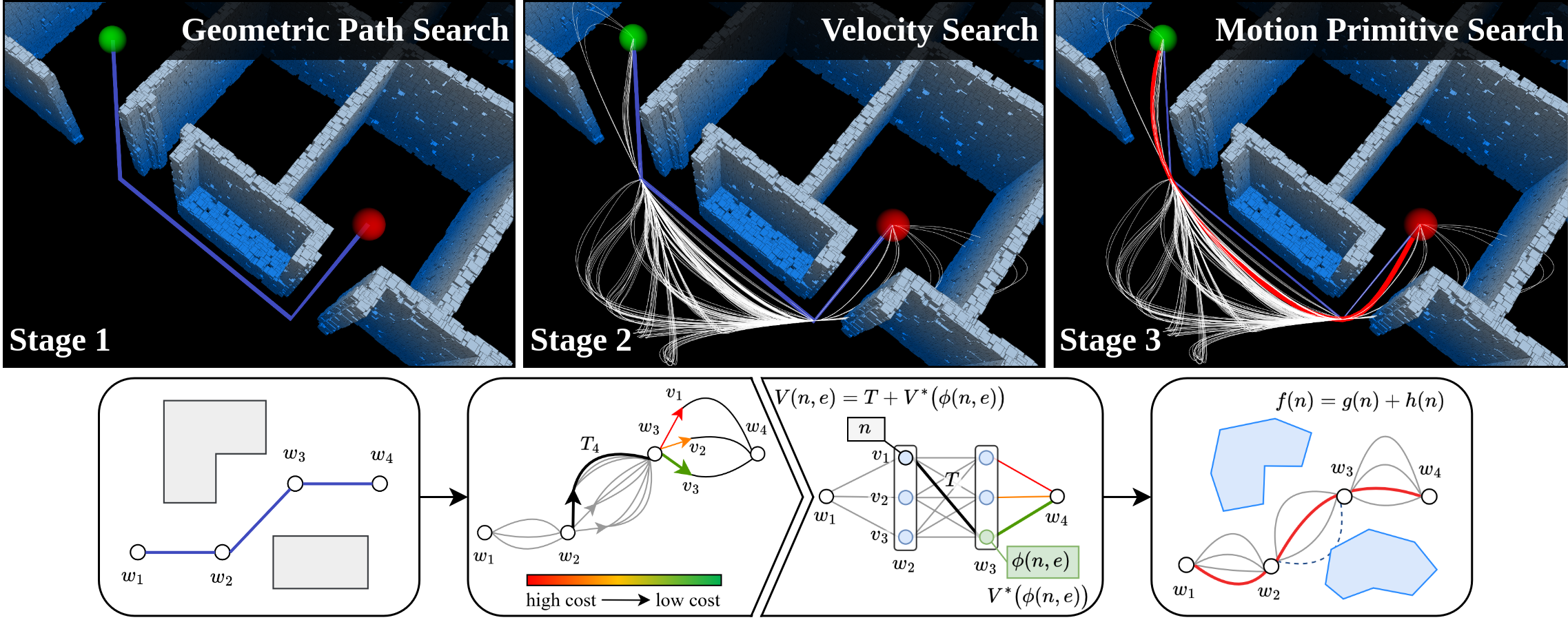}
  \caption{System architecture describing the three planning stages. Stage 1: A sparse geometric path is found via jump point search through the voxelized environment. Stage 2: A velocity state is then introduced at each waypoint and dynamic programming is used to recursively solve for the cost-to-go at each node. Stage 3:  A full motion primitive search informed by the previous stages is performed, and checks for collisions are completed to yield the final trajectory.}
  \label{fig:full_sys_arch}
\vskip -0.2in    
\end{figure*}

\textit{Minimum-Time Double Integrator:} Given an initial state $({s}_0, \, {{v}}_0) \in \mathbb{R}^2$ and desired final state $({s}_f, \, {{v}}_f) \in \mathbb{R}^2$, the minimum-time double integrator optimal control problem is 
\begin{align}
\min_{{u}} \quad & J = T \label{eq:double_min_time}
 \\
\text{s.t.} \quad &\ddot{{s}} = {u}, ~ |{u}|  \leq u_{max} \notag \\
& {s}(0) = {s}_0, \ {{v}}(0) = {{v}}_0 \notag \\
& {{s}}(T) = {{s}}_f,  \ {{v}}(T) = {{v}}_f, \notag
\end{align}
where the final time $T$ is free.
From Pontryagin's minimum principle, the optimal control solution is known to have a bang-bang control profile.
The switching times can be efficiently computed by solving a quadratic equation.
While each coordinate axis can be solved independently, all coordinate axis trajectories must share the same execution time. The minimum-time problem is therefore solved in two steps. First, the minimum-time horizon is found for each axis, yielding $T_x,\, T_y, \, T_z$. The overall trajectory duration is then defined as the limiting minimum-time horizon, $T^* = \text{max}\{T_x,T_y,T_z\}$. 
Next, the non-limiting axes are re-solved as fixed-time problems by applying the limiting minimum-time horizon $T^*$ as a known variable. We preserve the bang-bang control structure and solve a quadratic equation for the new reduced control bound, $\bar{u}\leq u_{max}$, which satisfies the longer fixed-time. 

\textit{Linear Quadratic Minimum-Time $p$-th Order Integrator:} 
Smooth trajectories can be generated by solving the linear quadratic minimum-time (LQMT) optimal control problem,
\begin{align}
    \min_{T,\,{u}} \quad & J = \rho \, T + \int_{0}^{T} {u}^2 \, dt \label{eq:triple_min_effort} \\
    \text{s.t.} \quad &{{s}}^{(p)} = {u} \notag \\
    & {{s}}(0) = {{s}}_0, \, {{v}}(0) = {{v}}_0, \dots, \, s^{(p\!-\!1)}(0) = s^{(p\!-\!1)}_0 \notag \\
    & {{s}}(T) = {{s}}_f, \, {{v}}(T) = {{v}}_f, \, s^{(k\!-\!1)}(T) \ \text{free for $3 \leq k \leq p$} \notag
\end{align}
where $\rho>1$ penalizes the final time.
The final time $T$ and all terminal states except position and velocity are free.  
The optimal trajectory is a polynomial in time so the cost functional can be expressed analytically in terms of $T$ and the known boundary conditions.
The final time can be found efficiently using a root-finding algorithm such as QR algorithm \cite{Demmel97:Applied}.
State constraints are omitted from \cref{eq:triple_min_effort} as it is more efficient to prune many candidate trajectories once the final time is known, as discussed in \cref{sec:pruning}.

\section{METHODOLOGY}

\begin{algorithm}[t]
    \small
	\SetAlgoLined
	\textbf{input:} $\mathcal{P} \leftarrow$ point cloud, $n_{s}\leftarrow$ start, $n_{g}\leftarrow$ goal; \, \\
	\textbf{output:} $\mathbf{s}^*(t)$; \\
	
	\BlankLine
	
    \footnotesize {\tcp{extract waypoints and path features}} 
    $\mathbf{w}, \mathbf{q}, \mathbf{\mathcal{H}} \leftarrow$ getGeometricPath($\mathcal{P}$, $n_{s}$, $n_{g}$);\\
    $\mathcal{G}\leftarrow$ buildVelocityGraph($\mathbf{w},\mathbf{q}, \mathbf{\mathcal{H}}$); \\ 
    \footnotesize {\tcp{get heuristic from recursive cost-to-go}} 
    $h(n) \leftarrow$ dynamicProgramming($\mathcal{G}, n_s, n_g$);\\
    $\mathcal{G}_{mp}\leftarrow$ buildFullStateGraph($\mathcal{G}$);\\ 
	\textbf{function} planPath($\mathcal{P}, \mathcal{G}_{mp}, n_s, n_g$): \\
        $n_{curr} = n_s$; \\
        \While{$n_{curr} \neq n_g$} {
        \footnotesize {\tcp{get node with lowest cost g(n)+h(n)}} 
        $n_{curr} = $ OPEN.pop(); \\
        CLOSED.insert($n_{curr}$); \\
        \If{$n_{curr} = n_g$}{ 
        break;}

        $\mathcal{E}_{n_{curr}} \leftarrow$ getSuccessors($n_{curr},\mathcal{G}_{mp}$); \\
        \For{$e \text{ in } \mathcal{E}_{n_{curr}}$}{
            \footnotesize {\tcp{collision and state constraint check}} 
            pruneMotionPrimitive($e, \mathcal{P}$); \\
            OPEN.insert($\phi(n,e)$); \\
        }
        }  
    $\mathbf{s}^*(t) \leftarrow$ getFinalMotionPrimitives($n_{curr}$, CLOSED);
	\BlankLine
    \caption{STITCHER Trajectory Planner}
	\label{alg:multilevel_search}
\end{algorithm}

STITCHER, detailed in \cref{alg:multilevel_search}, generates a full-state trajectory by \emph{stitching} collision-free, dynamically feasible trajectory segments together through graph search.
At its core, STITCHER searches over closed-form solutions, i.e., motion primitives, to optimal control problems of the form discussed above.
These solutions serve as a basis set for the solution space to \cref{eq:tpp}.
To achieve real-time performance, STITCHER utilizes a three stage planning process (see \cref{fig:full_sys_arch}). 
In Stage 1 (left), jump point search (JPS) is used to produce a sparse geometric path, i.e., waypoints, in the free space of the environment (line 3); this is standard in many planning frameworks.
In Stage 2 (middle), nodes representing sampled velocities at the waypoints are formed into a velocity graph where dynamic programming is used to compute the minimum time path from each node to the desired final state using a control-constrained double integrator model (lines 4-5). 
This step is critical for constructing an admissible heuristic to guide the full motion primitive search, and is one of the key innovations that enables real-time performance. 
Note that the optimal ``path" in velocity space is never used; computing the cost-to-go is the primary objective as it serves as an admissible heuristic for motion primitive search as shown in \cref{sec:DP}.
In Stage 3 (right), an A* search is performed using the heuristic from Stage 2 over motion primitives of a $p$-th order integrator where $p\geq2$ .
A greedy pre-processing step is used to construct a compact motion primitive graph (line 6), ensuring the search remains real-time. 
At this stage, position and all higher-order derivatives are considered, yielding a full state trajectory that can be tracked by the system (lines 7-21).
The remainder of this section expands upon each component of STITCHER.

\subsection{Stage 1: Forward Geometric Path Search}
STITCHER requires a sequence of waypoints that essentially guides the motion primitive search by limiting the size of the search space. 
This can be done by generating a collision-free geometric path (see \cref{fig:full_sys_arch} left) through the environment with  JPS or any other path-finding algorithm (e.g., Dijkstra, A*, RRT*) where the environment is represented as a 3D voxel grid with each grid cell containing occupancy information.
Let the collision-free, geometric path generated by a discrete graph search algorithm be composed of points $\mathcal{O}= \{\mathbold{o}_1,\mathbold{o}_2, ..., \mathbold{o}_H\}$ where $\mathbold{o}_i \in \mathbb{R}^3$. 
The set of points $\mathcal{O}$ is further pruned to create a sparse set of waypoints $\mathcal{W} = \{\mathbold{w}_1, \mathbold{w}_2, ..., \mathbold{w}_N\}$ where $N \leq H$ and $\mathbold{w}_i \in \mathbb{R}^3$. 
Sparsification is done by finding the minimal set of points in $\mathcal{O}$ that can be connected with collision-free line segments.
The geometric path search is used in line 3 of \cref{alg:multilevel_search}.

\subsection{State 2: Backward Velocity Search} \label{sec:phase_2_vel_graph}
The ordered waypoint set $\mathcal{W}$ found in Stage 1 only provides a collision-free geometric path through the environment.
In other words, the velocity, acceleration, and higher-order states necessary for tracking control are not specified. 
We propose creating a velocity graph (see \cref{fig:full_sys_arch} middle) where each node in the graph is defined by a position and velocity.
The positions are restricted to waypoint locations and $M$ velocities are sampled at each waypoint. 
More explicitly, for each waypoint $\mathbold{w}_i \in \mathcal{W}$, we sample a set of velocities $\mathcal{V} = \{\mathbold{v}_1, ..., \mathbold{v}_M\}$, where $\mathcal{V}$ is composed of candidate velocity magnitudes $\mathcal{V}_m$ and directions $\mathcal{V}_d$. 
The choice of $\mathcal{V}_m$ and $\mathcal{V}_d$ can impact the STITCHER's performance in terms of path optimality and computational complexity; this point will be discussed in more detail in \cref{sec:results}.

With the ordered waypoint $\mathcal{W}$ and sampled velocity $\mathcal{V}$ sets, we create a velocity graph $\mathcal{G} = (\mathcal{N}, \mathcal{E})$, where node $n \in \mathcal{N}$ is a given position and sampled velocity, i.e., $n = (\mathbold{w}_i,\,\mathbold{v}_j)$ with $\mathbold{w}_i \in \mathcal{W}$ and $\mathbold{v}_j \in \mathcal{V}$, and edge $e \in \mathcal{E}$ is the \emph{double integrator control-constrained minimum-time} motion primitive $\mathbold{r}(t)$ from \cref{eq:double_min_time} that connects neighboring nodes. 
Recall that the solution to \cref{eq:double_min_time} is fully determined by having an initial and final position and velocity pair which is precisely how each node in $\mathcal{N}$ is defined.
At this stage, collision and state constraints are not enforced to prevent candidate trajectories from being eliminated prematurely.

We recursively compute and store the ranked list of cost-to-go's $V_d:\mathcal{N}\times \mathcal{E} \rightarrow \mathbb{R}_+$ for each node $n \in \mathcal{N}$ and all connecting edges $e \in \mathcal{E}_n$ of $n$ where
\begin{align}
\label{eq:bellman}
    V_d(n,e) = \ell(n,e) + V_d^*\big(\phi(n,e)\big) ~~ \forall  e\in \mathcal{E}_n,
\end{align}
with the optimal cost-to-go $V_d^*(n) = \min_{e\in \mathcal{E}_n} V_d(n,e)$, the cost of taking edge $e$ from node $n$
being $\ell(n,e)$, and the node reached by taking edge $e$ being $\phi(n,e)$. 
The cost of taking an edge is given by $\ell(n,e) = T^*_d(n,e)$, where $T^*_d(n,e)$ is the minimum-time of trajectory $\mathbold{r}(t)$ connecting the states of node $n$ to the states of $\phi(n,e)$. 
Minimizing \eqref{eq:bellman} is the well-known Bellman equation, which is guaranteed to return the optimal cost-to-go.
In \cref{sec:DP}, we prove $V_d^*(n)$ for each node in a graph $\mathcal{G}$ is an admissible heuristic for an A* search over a broad class of motion primitives.
Building and searching the velocity graph are shown in lines 4-5 of \cref{alg:multilevel_search}.

\subsection{Stage 3: Forward Motion Primitive Search} \label{sec:stage3}
The cost-to-go's computed in Stage 2 for the sampled velocities at each waypoint serve as an admissible heuristic (formally defined later in \cref{def:admissable}) that guides an efficient A* search over motion primitives.
The motion primitives can be generated using any chain of integrators model of order at least two so long as i) the initial and final position and velocities match those used to construct the velocity graph $\mathcal{G}$ and ii) the allowable acceleration is maximally bounded by $u_{max}$ given in \cref{eq:double_min_time}.
It is important to note that the bound on allowable acceleration can be easily satisfied with user defined $u_{max}$ or  simply by applying the box constraint $\|\boldsymbol{a}\|_{\infty}\leq u_{max}$. 
The motion primitive search graph is denoted as $\mathcal{G}_{mp} = (\mathcal{N}_{mp},\, \mathcal{E}_{mp})$ where $\mathcal{N}_{mp}$ is the set of nodes, each corresponding to a state vector, and $\mathcal{E}_{mp}$ is the set of edges, each corresponding to a motion primitive that connects neighboring nodes.
A* search is used to meet real-time constraints where the search minimizes the cost $f(n) = g(n) + h(n)$ where $n\in \mathcal{N}_{mp}$ is the current node, $g: \mathcal{N}_{mp} \rightarrow \mathbb{R}_+$ is the cost from the start node $n_s$ to node $n$, and $h: \mathcal{N}_{mp} \rightarrow \mathbb{R}_+$ is the estimated cost from the current node $n$ to the goal node $n_g$.  
In the context of optimal control, $g$ is the cost accrued, i.e., the cost functional $J$, for a path from $n_s$ to $n$ whereas $h$ is the estimated cost-to-go, i.e., the estimated value function $V^*$, from $n$ to $n_g$.
In this stage, collision and state constraints are checked for each candidate motion primitive to ensure safety.
The methodology for both is discussed in \cref{sec:pruning}.
Each step of the motion primitive A* search is shown in lines 6-21 of \cref{alg:multilevel_search}.

\subsection{Pruning Infeasible \& In-Collision Motion Primitives}
\label{sec:pruning}

STITCHER guarantees safety by pruning motion primitives from the final search that violate constraints or are in collision.
For state and actuator constraints, many optimization-based planning approaches approximate the true physical constraints of the system with simple convex constraints, e.g., $\Vert \mathbold{{v}}\Vert_\infty \leq v_{max},~\Vert \mathbold{{a}}\Vert_\infty \leq a_{max},$ etc., to reduce computational complexity.
When polynomials are used to represent the optimal trajectory, imposing a convex hull constraint on the polynomial is one method for enforcing such state constraints \cite{Zhou19:Robust, Tordesillas22:FASTER}.
However, many of these approximations are made only to simplify the resulting optimization problem and do not accurately reflect the actual physical constraint, which can lead to conservatism.
STITCHER has the freedom to use a variety of methods to enforce state and actuator constraints. For the examples shown in this work, we uniformly sample candidate trajectories in time to check for constraint violations as it was found to be effective and efficient.
Sampling avoids mistakenly eliminating safe trajectories, and the observed computation time under the current discretization was better than using convex hulls.
Critically, sampling allows for the inclusion of more complex constraints, such as those that couple multiple axes. 
Examples of such constraints for a VTOL vehicle are given in \cref{sec:exp_constr}, such as individual motor thrust constraints, tilt angle, and linear and angular velocity magnitude, which couple multiple axes and depend nonlinearly on the desired position trajectory and its derivatives.

\begin{figure}[t]
  \centering
  \includegraphics[width=\columnwidth]{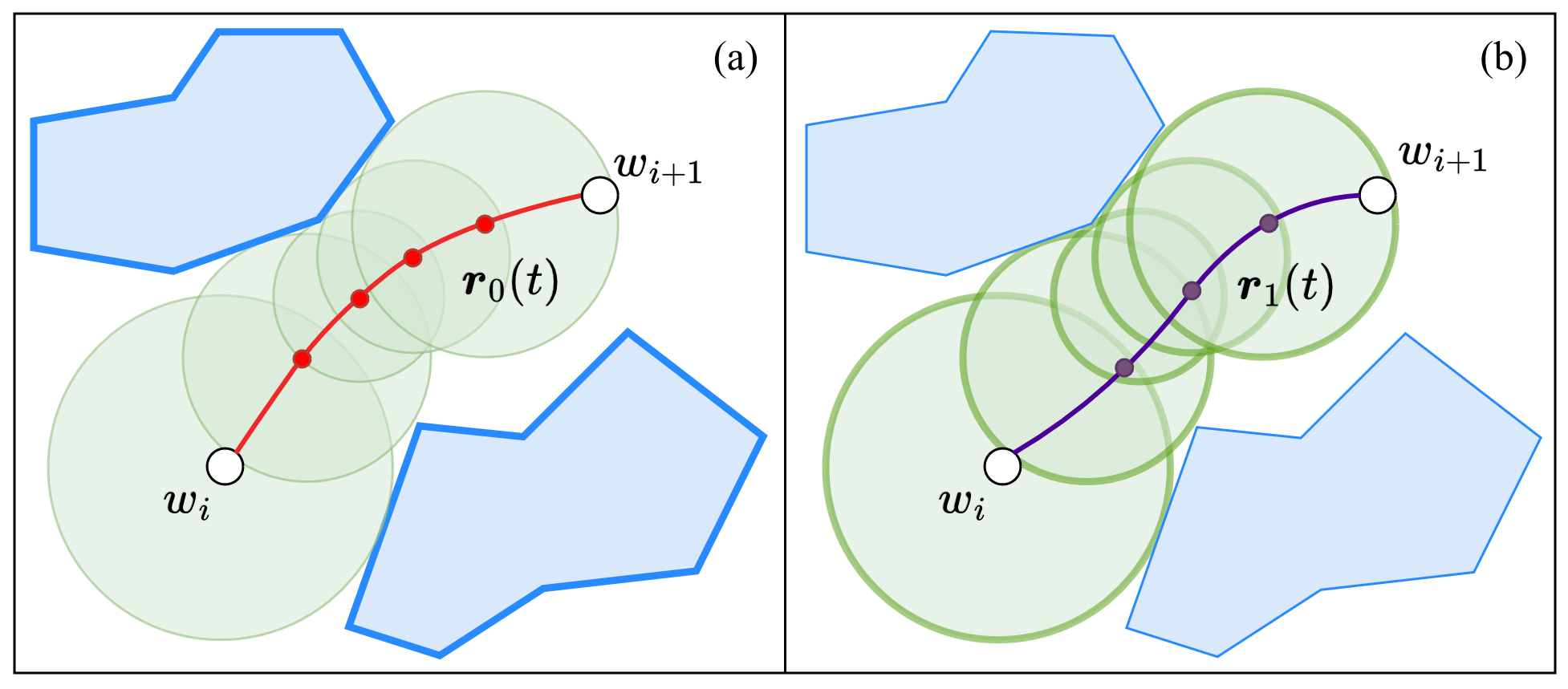}
  \caption{Removing redundant collision checks. (a): Motion primitive $\mathbold{r}_0(t)$ checks for collisions using \cite{Lopez17:Aggressive}. (b): Sampled points of $\mathbold{r}_1(t)$ are checked to lie within obstacle-free regions derived from $\mathbold{r}_0(t)$ calculations.}
  \label{fig:smart_collision}
\end{figure}

An efficient collision checking strategy was devised by constructing a safe set of spheres resulting from a sampling-based collision checking approach developed in \cite{Lopez17:Aggressive} (see \cref{fig:smart_collision}). 
The core idea from \cite{Lopez17:Aggressive} is that a trajectory can be intelligently sampled for collisions by estimating the next possible ``time-of-collision" along the trajectory by combining obstacle proximity and the vehicle's maximum speed. 
Leveraging this idea, further computation time savings can be achieved by storing and reusing nearest neighbor queries. 
\cref{alg:collision_check} details our strategy which takes in a \textit{k}-d tree data structure filled with points from a point cloud $\mathcal{P}$ of the environment.
For the first candidate motion primitive connecting two successive waypoints, we use the strategy from \cite{Lopez17:Aggressive} to intelligently sample for collisions while also storing the resulting set of safe, obstacle-free spheres $\mathcal{S}$, defined by center and radius vectors, ${\mathbf{c}}$ and $\mathbf{R}$ (line 4). 
For subsequent motion primitives between the same waypoint pair, a nearest neighbor query is only done if the primitive is expected to leave the initial set of obstacle-free spheres. 
For a point found to be within a certain sphere (lines 12-15), the next possible ``time-of-collision" is when the trajectory intersects the edge of the sphere, which can estimated by assuming the trajectory is emanating radially from the center of the sphere at maximum velocity (line 18).
The process is repeated until the final time horizon $T$ is reached.
Unlike spherical safety corridors, our safe set is only used as a means to avoid repeated calculation, and allows for on-the-fly addition of collision-free spheres.
In other words, our approach does not restrict solutions to remain within convex sets centered along the geometric path.
STITCHER thus has the flexibility to create and check candidate trajectories without being restricted to pre-defined safety spheres. 

\begin{algorithm}[t]
    \small
	\SetAlgoLined
	\textbf{input:} $\mathcal{T} \leftarrow \text{\textit{k}-d tree}, $\, $\mathbold{r}(t) \leftarrow \text{motion primitive}, \,$ \\
	\textbf{output:} bool \textit{collision} \\
	
	\BlankLine
	
    \If{$\mathcal{S} = \emptyset $}{ 
         \footnotesize {\tcp{initial collision check using \textit{k}-d tree}}
        $collision,~\mathcal{S} \leftarrow$ collisionCheckMap($\mathbold{r}(t), \, \mathcal{T}$); \\
        \Return{\textit{collision}}
    }
    \BlankLine
    $\tau \leftarrow 0$; \\
    $d_{min} \leftarrow \infty$; \\
    \While{$\tau \leq T$} {
        $\mathbf{d} \leftarrow$ calcDistToSphereCenters($\mathbf{c},\mathbold{r}(\tau)$) \\
        \For{$i = 1$ to $|\mathcal{S}|$}{
        \If{$d_i < R_i$ \& $d_i \leq d_{min}$ }{$d_{min} \leftarrow d_i$; \\ $k \leftarrow i$;}
        }

        \eIf{$d_{min} < \infty$}{\footnotesize {\tcp{update sample time}} $\tau \leftarrow \tau + (R_k-d_k)/v_{max}$; }{\footnotesize {\tcp{point outside spheres, use \textit{k}-d tree}} 
            $collision,~\mathcal{S} \leftarrow$ collisionCheckMap($\mathbold{r}(t), \, \mathcal{T}$); \\ 
        }
    }
    
    \Return{\textit{collision}}
    
	\BlankLine
    \caption{Collision Check}
	\label{alg:collision_check}
\end{algorithm}

\begin{figure*}[t]
 \centering
 \includegraphics[width=\textwidth]{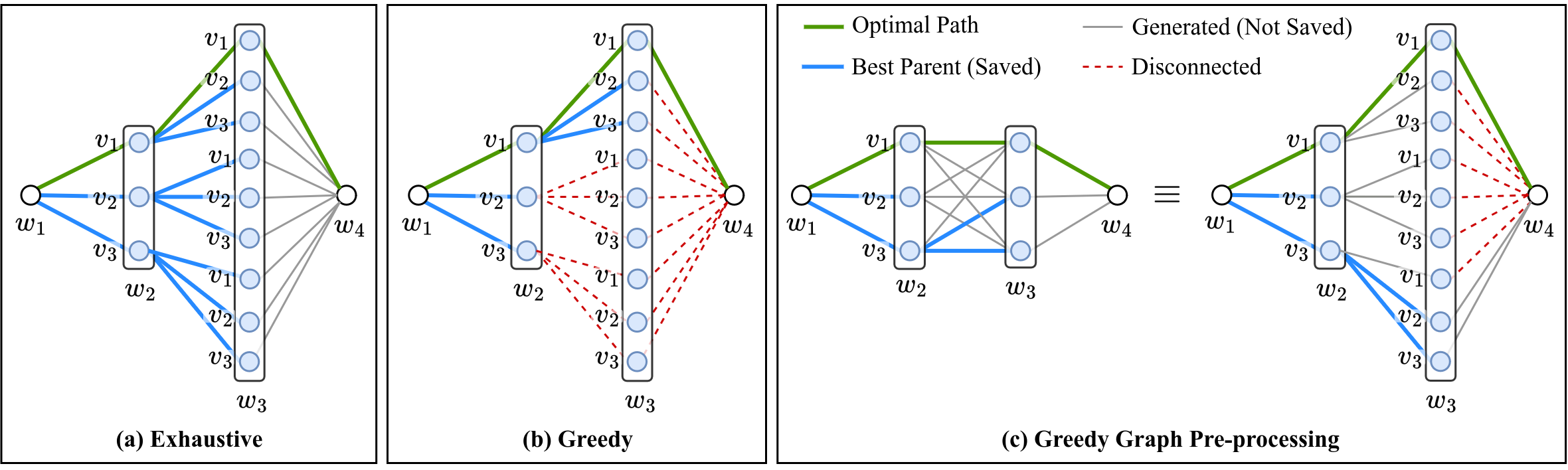}
\caption{Difference in motion primitive graphs with free terminal accelerations employing an (a): exhaustive search, (b): greedy search, and (c): greedy graph pre-processing step. Green edges constitute the optimal path, blue edges are the least cost parent saved in memory, grey edges are those evaluated but not saved and red dashed edges are connections that no longer exist. Greedy pre-processing (c) leads to some disconnections for graphs with over 3 waypoints, but maintains more connections than a greedy search (b), and does not suffer from exponential edge growth like the graph of an exhaustive search (a).}
\label{fig:greedy_graph} 
\vskip -0.2in 
\end{figure*}

\subsection{Motion Primitive Search Graph with Triple Integrator}
In many applications, a triple integrator model for generating motion primitives is sufficient because the resulting trajectory is smooth enough for most aerial vehicles to track as discontinuities in jerk typically do not severely degrade tracking. This was verified through hardware experiments discussed in \Cref{sec:hardware_results}. 
Motion primitives in our formulation \cref{eq:triple_min_effort} are derived imposing a free terminal acceleration.

Constructing a motion primitive search graph where nodes are a waypoint-velocity-acceleration tuple drastically increases both computation and memory consumption as the graph size depends on both the number of sampled velocities and accelerations. If the acceleration at each node, i.e., the final acceleration, ${\mathbold{a}}_f$, is free, the number of edges grows exponentially with respect to the number of waypoints (See \cref{fig:greedy_graph}a). 
Our formulation employs a greedy pre-processing step (\cref{fig:greedy_graph}c) in which the motion primitive search graph $\mathcal{G}_{mp}$ is identical in size to the velocity graph $\mathcal{G}$ (graph size detailed in Section \ref{sec:vel_graph_size}).
This formulation offers an advantage in terms of computational efficiency, as a full state trajectory is generated while the graph size is restricted by only the number of sampled velocities. 
Excluding acceleration information from graph creation assumes that the optimal stitched trajectory is only weakly dependent on acceleration at each waypoint. 
Specifically, instead of storing multiple acceleration states for each waypoint-velocity pair as in \cref{fig:greedy_graph}a, only one acceleration state is retained. Each node retains a single acceleration value corresponding to the free terminal acceleration of the best parent edge (shown in blue in \cref{fig:greedy_graph}c). 
The best parent edge is determined by the cost accrued and estimated cost-to-go as detailed in \cref{sec:stage3}.
This acceleration is then used as an initial condition for all outgoing motion primitives generated from the node. Consequently, portions of the exhaustive graph that would be generated from alternative accelerations are not explored, making the pre-processing step greedy.
It is important to note the difference between a greedy algorithm (\cref{fig:greedy_graph}b) and our greedy graph pre-processing step (\cref{fig:greedy_graph}c). While the pre-processing step does limit edges that could be generated in the exhaustive graph, it maintains more than a greedy algorithm. 
Further comparisons of the relative solution cost and computation speed using STITCHER's greedy pre-processing step versus the exhaustive search were conducted in Section \ref{sec:greediness_benchmarking}.

\section{THEORETICAL ANALYSIS}
In this section we prove STITCHER has bounded time and memory complexity by showing the velocity and motion primitive graphs are finite. 
We also show STITCHER is complete and optimal by proving the heuristic used in the motion primitive search is admissible.

\subsection{Velocity Graph Complexity} \label{sec:vel_graph_size}
The following proposition proves the size of the velocity graph $\mathcal{G}$ is finite and solely depends on the number of waypoints and sampled velocities; a property that also holds for the motion primitive graph $\mathcal{G}_{mp}$ by extension.
This result is critical as a finite graph yields \emph{known time complexity} for the motion primitive search.
In other words, an upper bound can be placed on the computation time of the planner given known quantities. 
This is in contrast to optimization-based methods where the time complexity depends on the number of iterations required to converge---which cannot be known \textit{a priori}---so the time to compute a trajectory via optimization does not have an \textit{a priori} bound.
 
\begin{proposition} \label{prop_bwd_size}
For $N$ waypoints and $M$ sampled velocities, the number of nodes $|\mathcal{N}|$ and edges $|\mathcal{E}|$ in graph $\mathcal{G}$ is
\begin{align}
|\mathcal{N}| &= (N-2)M + 2, \label{eq:bwd_total_nodes}\\
|\mathcal{E}| &= (N-3)M^2 + 2M ~~ \text{for}~N > 2.
\label{eq:bwd_total_edges}
\end{align}
\end{proposition}
\begin{proof}
Using \cref{fig:full_sys_arch} (middle), the start and goal nodes contribute 2 nodes to the graph $\mathcal{G}$. 
For intermediate waypoints, given $M$ sampled velocities, there are $M$ nodes per waypoint.
As a result, $|\mathcal{N}| = (N-2)M + 2$ which is \cref{eq:bwd_total_nodes}.
For each edge, we consider the transition to successive waypoints. 
Ignoring the trivial $N = 2$ case where $|\mathcal{E}| = 1$, there are $M$ connections between the start node and next waypoint, which also has $M$ nodes. 
The same applies for connecting waypoint $w_{N-1}$ to the goal node, resulting in a total of $2M$ edges. 
For all other intermediate waypoint pairs, $M$ nodes connect to $M$ nodes at the next waypoint so there are $M^2$ edges. 
The total number of edges is then \cref{eq:bwd_total_edges}.
\end{proof}

\begin{corollary}
\label{cor:lqmt}
    The size of the motion primitive graph $\mathcal{G}_{mp}$ using Linear Quadratic Minimum Time (LQMT) motion primitives with free terminal acceleration for a triple integrator is identical to the velocity graph $\mathcal{G}$.
\end{corollary}

\begin{proof}
    The proof is immediate since the terminal acceleration is free so $N$ and $M$ are identical for both graphs.
\end{proof}

\begin{remark}
    \Cref{cor:lqmt} can be generalized to any motion primitive search graph where the primitives are solutions to an optimal control problem with the dynamics being a chain of integrators and all terminal state derivatives of second order or higher are free.
    Note that this assumes the greedy graph pre-processing step still yields adequate performance.
\end{remark}

\subsection{Admissible Heuristic for Motion Primitive Search} \label{sec:DP}
It is well known that heuristics can be used to expedite searching through a graph by incentivizing the search to prioritize exploring promising nodes.
For example, in A* search, the next node explored is selected based on minimizing the cost $f(n) = g(n) + h(n)$, where $g$ is the stage cost to get from the start node $n_s$ to node $n$, and $h$ is a heuristic estimate of the remaining cost to reach the goal node $n_g$.
A* search is guaranteed to find an optimal solution so long as the heuristic function $h$ is admissible (see \cref{def:admissable}) \cite{Russell16:Artificial}.
Below, we prove the cost-to-go $V^*$ for each node in the velocity graph $\mathcal{G}$ calculated in Stage 2, i.e., the minimum-time to goal for a double integrator, is an admissible heuristic for an A* search over motion primitives of any higher-order chain of integrators.
 
\begin{definition}[{\cite{Russell16:Artificial}}] 
\label{def:admissable}
A function $h : \mathcal{N} \rightarrow \mathbb{R}$ is an admissible heuristic if for all $n \in \mathcal{N}$ then $h(n) \leq h^*(n)$, 
where $h^*$ is the optimal cost from $n$ to the goal node $n_g$.
\end{definition}

\begin{proposition}
\label{prop:c2g_opt}
    Consider the optimal control problem
    \begin{align}
    \label{eq:gen_opt}
            \min_{T,\,\mathbold{u}} \quad & J =  \rho \, T + \int^T_0 q(\mathbold{r},\mathbold{{v}},\dots,\mathbold{u}) \, dt  \\ 
            \mathrm{s.t.} \quad &  \mathbold{r}^{(p)} = \mathbold{u}, \ c(\mathbold{a}) \leq 0 \notag \\
            & \mathbold{r}(0) = \mathbold{r}_0, \, \mathbold{v}(0) = \mathbold{v}_0, \, \dots , \mathbold{r}^{(p\!-\!1)}(0) = \mathbold{r}_0^{(p\!-\!1)} \notag\\
            & \mathbold{r}(T) = \mathbold{r}_f, \, \mathbold{v}(T) = \mathbold{v}_f, \,  \mathbold{r}^{(k\!-\!1)}(T) \ \mathrm{free} \ \mathrm{for} \ 3 \leq k \leq p  \notag
    \end{align}
    where $q$ is a positive definite function, the system is at least second order ($p\geq 2$), and the position and velocities boundary conditions are identical to those of \cref{eq:double_min_time}, with all other boundary constraints free to specify. 
    If $u_{max}$ in \cref{eq:double_min_time} is the maximum possible acceleration achievable in a given axis imposed by $c(\mathbold{a}) \leq 0$, then the optimal cost-to-go $V^*$ from the initial conditions for \cref{eq:gen_opt} satisfies $V^*  \geq  \rho \, T_d^*$ where $T_d^*$ is the optimal final time for \cref{eq:double_min_time}.
\end{proposition}
\begin{proof}
    First, consider the case when $p=2$. 
    For a given axis, if $u_{max}$ is chosen so that it exceeds the allowable acceleration imposed by $c(\boldsymbol{a}) \leq 0$, e.g., $u_{x,max} \geq \max_{a_x}c({\boldsymbol{a}})$ (see \cref{fig:coupled_constr}), then the optimal final time $T^*$ for $\cref{eq:gen_opt}$ will always be greater than that of \cref{eq:double_min_time} even when $q = 0$.
    Specifically, when $q = 0$, one can show the optimal final time for \cref{eq:double_min_time} increases as $u_{max}$ decreases. 
    Moreover, $T^*_d$ for \cref{eq:double_min_time} is guaranteed to exist and be unique \cite{Kirk04:Optimal}.
    Hence, by appropriate selection of $u_{max}$, we can ensure $T^* \geq T^*_d$ always, where equality holds when $q = 0$ and $c(\boldsymbol{a})$ is a box constraint. 
    If $q \neq 0$, then it immediately follows that $T^* > T^*_d$ because $q$ is positive definite by construction.
    Now consider the case when $p > 2$.
    We can deduce $V^*  >  \rho \, T^*_d$ by contradiction.
    Specifically, assume $T^* = T_d^*$ for $p > 2$.
    This would require $\boldsymbol{a}$ to be discontinuous in order to match the bang-bang acceleration profile of \cref{eq:double_min_time}.
    However, \cref{eq:gen_opt} is a continuous-time linear system that will not exhibit discrete behaviors, e.g., jumps, so it is mathematically impossible to generate an optimal control sequence where the acceleration profiles for \cref{eq:gen_opt,eq:double_min_time} will be identical. 
    It can then be concluded $V^* > \rho \, T^*_d$ for $p > 2$.
    Therefore, $V^* \geq \rho \, T^*_d$ for $p \geq 2$, as desired.
\end{proof}

\begin{remark}
    \cref{prop:c2g_opt} also holds when inequality state or actuator constraints in \cref{eq:gen_opt} are present, and when the terminal desired states are specified rather than free.
\end{remark}

The main result of this section can now be stated. 

\begin{theorem}
\label{theorem:admissible}
    The optimal cost-to-go for the minimum-time input-constrained double integrator optimal control problem \cref{eq:double_min_time} is an admissible heuristic for motion primitive search where the primitives are solutions to the optimal control problem of the form \cref{eq:gen_opt}.
\end{theorem}

\begin{proof}
    Let $\mathcal{G} = (\mathcal{N},\mathcal{E})$ be a graph with nodes being sampled velocities at waypoints and edges being the time-optimal trajectories using an input-constrained double integrator.
    Further, let $\mathcal{G}_{mp} = (\mathcal{N}_{mp},\mathcal{E}_{mp})$ be a graph with nodes being sampled velocities, accelerations, etc. at waypoints and edges being trajectories that are solutions to \cref{eq:gen_opt}. 
    Using the Bellman equation, the optimal cost-to-go $V_{mp}^*(n)$ for any $n \in \mathcal{N}_{mp}$ can be computed recursively.
    Using \cref{prop:c2g_opt}, $V_{mp}^*(n) \geq V_d^*(n^\prime)$ by induction where $V_d^*$ is the optimal cost-to-go for the minimum-time input-constrained double integrator with $n^\prime \in \mathcal{N}$. 
    Recognizing $\mathcal{N} \subseteq \mathcal{N}_{mp}$, $V_d^*(n^\prime)$ can be rewritten as $V_d^*(n)$.    
    Setting $h^*(n) = V^*_{mp}(n)$ and $h(n) = V_d^*(n)$, it can be concluded $h(n) \leq h^*(n)$.
    Therefore, by \cref{def:admissable}, the optimal cost-to-go computed for $\mathcal{G}$ is an admissible heuristic for the motion primitive search over $\mathcal{G}_{mp}$.
\end{proof}

The importance of \cref{theorem:admissible} follows from the well-known result that searching a graph with an admissible heuristic is \emph{guaranteed} to return the optimal path through the graph \cite{Russell16:Artificial}, and can significantly improve search efficiency because not every node in the graph has to be explored.
The effectiveness of the proposed heuristic both in terms of path quality and search times is analyzed in \cref{sec:heuristic_benchmark}.

\subsection{Summary}
The preceding theoretical analysis identifies two significant contributions to the computational efficiency of STITCHER. First, we show that the motion primitive search graph is finite. This result is significant because the graph size remains bounded and only scales with the number of velocity samples, as compared to an exhaustive motion primitive search (see \cref{fig:greedy_graph}a), which grows exponentially with respect to the number of waypoints. Next, we develop an admissible heuristic to reduce search effort while maintaining the optimality guarantees of A* search. Importantly, the heuristic is environment agnostic because it only relies on the estimated minimum-time cost between states and does not take into account obstacle or perception data. These contributions reduce the size of the search graph and the number of node expansions, providing a theoretical foundation for experimental results in \cref{sec:results}.

\section{SIMULATION RESULTS}
\label{sec:results}
\begin{figure}[t]
 \centering
 \includegraphics[width=\columnwidth]{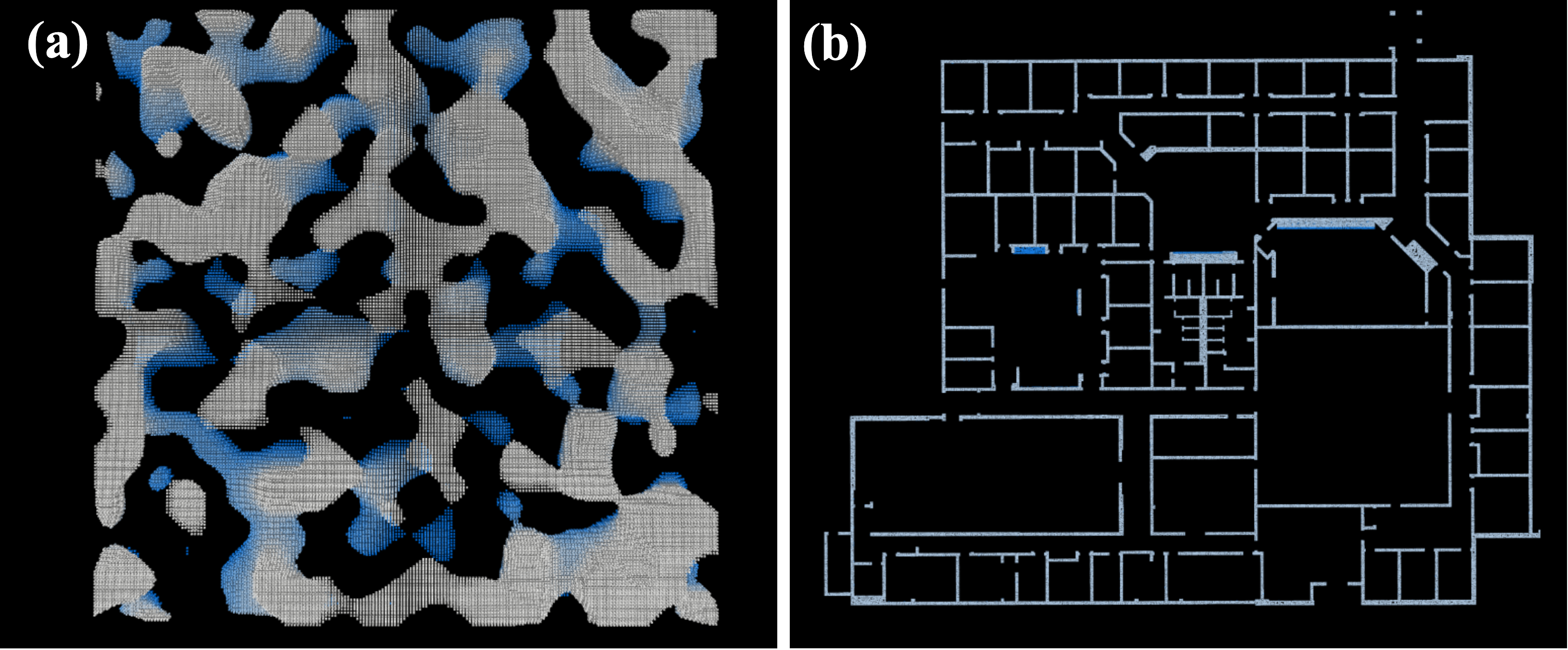}
\caption{Simulation test environments. (a): Willow Garage environment. (b): Perlin Noise environment.}
\label{fig:env_type} 
\vskip -0.1in
\end{figure}

This section contains an analysis of STITCHER. 
First, the state and actuator constraints enforced throughout the experiments are defined.
Second, we conduct a parameter sensitivity study to determine a suitable sampled velocity set (direction and speed) using a modified version of Dijkstra's algorithm that has access to a dense set of velocities and no constraints on computation time.
Third, STITCHER is compared to a non-greedy algorithm variant to characterize the impact of our greedy graph pre-processing step on solution cost and computation.
Fourth, we investigate the effectiveness of the heuristic proposed in \cref{sec:DP} in reducing the number of edges generated by STITCHER.
Fifth, we characterize the average computation time of the different components that make up STITCHER.
Sixth, a study of STITCHER constraining individual motor forces is presented to demonstrate the flexibility of our method in adhering to complex actuator constraints. 
Lastly, STITCHER is compared to three optimization-based modern planners \cite{Tordesillas22:FASTER, Wang22:Geometrically, Kondo26:Mighty} capable of running in real-time.

Simulation experiments were run in a Perlin Noise environment and the Willow Garage environment,
both with a volume of approximately $50\times50\times5$ m (see \cref{fig:env_type}).
Geometric paths with $N =$ 4, 6, 8 waypoints were found for different start and end locations in each environment. 
For all experiments, we imposed the following constraints assuming agile drone flight and reasonable thrust limits: $f_{min} = 0.85 \text{ m/s}^2$, $f_{max} = 18.75 \text{ m/s}^2$, $\theta_{max} = 60^\circ$, $\omega_{max} = 6 \text{ rad/s}$, $v_{max} = 10 \text{ m/s}$, and a time penalty $\rho = 1000$.
All reported times are from tests conducted on an 11$^\mathrm{th}$ generation Intel i7 CPU.

\subsection{Experimental Constraints} \label{sec:exp_constr}
 Throughout all experiments, trajectories are required to satisfy state and actuator limits of a quadrotor system. These limits are enforced during the motion primitive search using the sampling-based pruning method described in \Cref{sec:pruning}. Specifically, the following constraints are imposed:

\begin{equation}
    \label{eq:state_constr}
    \begin{aligned}
        \text{Thrust Magnitude:} & ~~ 
        0 < f_{min} \leq \|\boldsymbol{f}\|_2 \leq f_{max} \\
        \text{Thrust Tilt Angle:} & ~~ \|\boldsymbol{f}\|_2 \cos(\theta_{max}) \leq \hat{\boldsymbol{e}}_z^\top \boldsymbol{f} \\
        \text{Linear Velocity:} & ~~ \|\boldsymbol{v}\|_2 \leq v_{max} \\ 
        \text{Angular Velocity:} & ~~ \|\boldsymbol{\omega}\|_2 \leq \omega_{max},
    \end{aligned}
\end{equation}
where $\mathbold{f}$ is the mass-normalized thrust vector, $\theta$ is the thrust tilt angle, and $\mathbold{\omega}$ is the angular velocity. 
Note that differential flatness can be leveraged to express the angular velocity constraint in terms of derivatives of position.
\Cref{fig:coupled_constr} depicts the achievable mass-normalized thrust of a VTOL vehicle given thrust and tilt constraints in \eqref{eq:state_constr}. 
The constraints are nonconvex making it difficult to include in real-time optimization-based planners without some form of relaxation, e.g., as in \cite{Acikmese07:Convex} for a double integrator, which is tight, or a more conservative relaxation. 
Even more direct system constraints which limit the maximum force exerted by individual motors are highly nonconvex functions of the flat variables and their derivatives. 
While the majority of this paper showcases results that enforce constraints \eqref{eq:state_constr}, we include a case study of STITCHER constraining individual thrusters in Section \ref{sec:indiv_motor}.  

\begin{figure}[t]
  \centering
  \includegraphics[width=0.9\columnwidth]{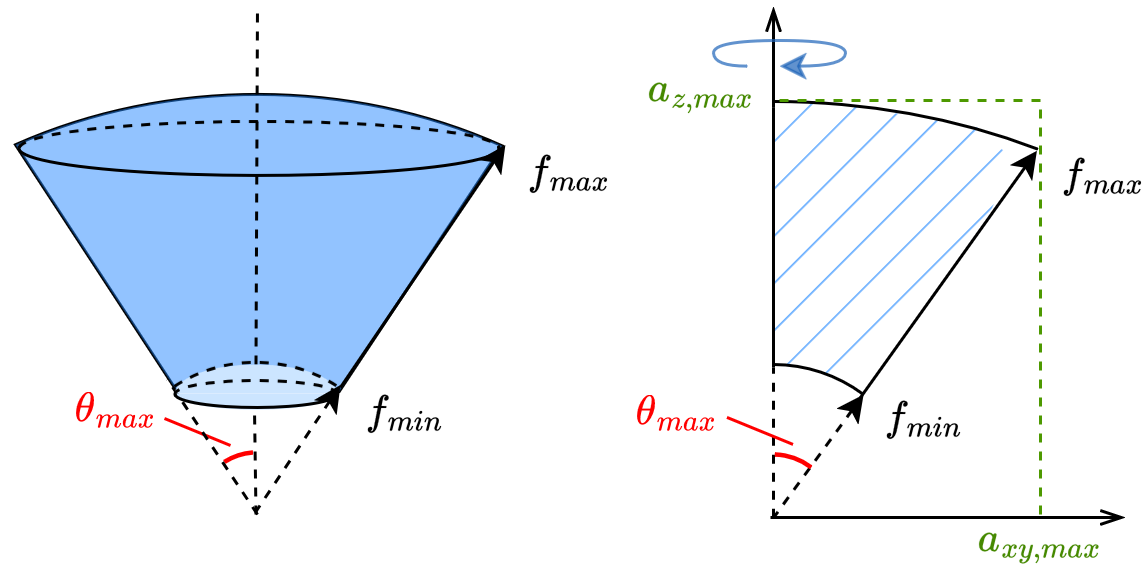}
  \caption{The achievable mass-normalized thrust (nonconvex) of a aerial VTOL vehicle with limits on minimum thrust $f_{min}$, maximum thrust $f_{max}$, and maximum tilt $\theta_{max}$.}
  \label{fig:coupled_constr}
\end{figure}

\subsection{Parameter Sensitivity Analysis: Sampled Velocity Set}
\label{sub:velocity_set}
STITCHER requires a discrete velocity set $\mathcal{V}$ composed of a set of magnitudes and directions. 
As shown in \cref{sec:phase_2_vel_graph}, the size of the velocity graph scales quadratically with the number of sampled velocities, so it is desirable to choose $\mathcal{V}$ to be small without degrading path quality. 
Hence, this section focuses on constructing a velocity sample set $\mathcal{V}$ that optimizes the trade-off between computation time (i.e., the time to generate a trajectory) and path quality as measured by execution time (i.e., the trajectory duration).
We conducted two studies to determine i) the predominate sampled velocity directions used by an offline version of Dijkstra's algorithm and ii) the trade-off between path cost and the number of sampled speeds. 
In the following comparison, Dijkstra's algorithm was modified to search over a dense set of velocities with 3611 terminal velocities sampled per waypoint; the method is referred to as Dense Dijkstra.
Dijkstra's algorithm is a complete and optimal search algorithm with respect to the set of actions \cite{Russell16:Artificial}, making it a suitable benchmark. 

\begin{figure}[t!]
  \centering
  \includegraphics[width=\columnwidth]{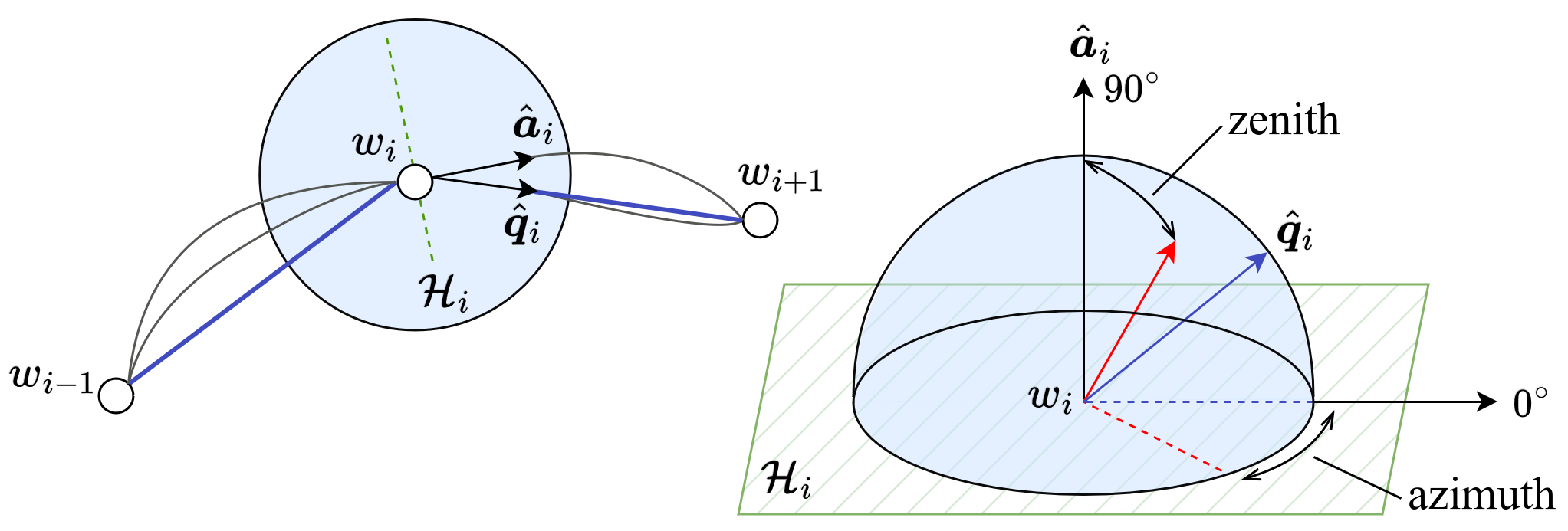}
  \caption{Velocity directions sampled at waypoint $w_i$, where $\hat{\mathbold{a}}_i$ is normal to the separating hyperplane $\mathcal{H}_i$ and $\hat{\mathbold{q}}_i$ is the heading toward $w_{i+1}$.}
  \label{fig:vel_sampling}
\vskip -0.1in   
\end{figure}

\begin{table}[t]
\caption{Frequency of Velocity Directions in Final Trajectory.}
\label{tab:freq_vel_dir}
\centering
\begin{tabular}{lcccccc}
\toprule
Zenith Angle  & 70$^\circ$ &80$^\circ$ & \multicolumn{2}{c}{90$^\circ$} & 100$^\circ$ & 110$^\circ$ \\
\cmidrule(r){2-7}
Frequency & 0\% & 13\% & \multicolumn{2}{c}{79\%} & 8\% & 0\%\\
\midrule
Azimuth Angle & -50$^\circ$ & -20$^\circ$ & -10$^\circ$ & 0$^\circ$ & 10$^\circ$ & 20$^\circ$ \\
\cmidrule(r){2-7}
Frequency & 4\% & 4\% & 17\% & 42\% & 29\% & 4\% \\
\bottomrule
\end{tabular}
\vskip 0.05in
\end{table}

\subsubsection{Sampled Direction} 
Dense Dijkstra was used to statistically identify velocity directions set $\mathcal{V}_d$ commonly employed in a variety of test cases.
Using \cref{fig:vel_sampling}, we define angles with respect to the hyperplane $\mathcal{H}$ with normal vector $\hat{\mathbold{a}}$ at the plane of symmetry between path segments connecting two waypoints.
The search was given velocity direction zenith and azimuth angles in the range [0$^\circ$, 180$^\circ$] and [-90$^\circ$, 90$^\circ$] sampled in 10$^\circ$ increments.

\cref{tab:freq_vel_dir} shows the frequency of velocity directions chosen by Dense Dijkstra across all motion primitives for six different path length trials in the Perlin Noise and Willow Garage environment. 
The velocities chosen by Dense Dijkstra align with the normal vector $\hat{\mathbold{a}}$ of the hyperplane 80\% of the time. 
From these results, sampling at the center and boundaries of a $20^\circ$ cone centered around a given $\hat{\mathbold{a}}$ will yield a suitable set of velocity directions.

\subsubsection{Sampled Speed} Using the velocity direction set $\mathcal{V}_d$ identified in the previous test, the performance of Dense Dijkstra was compared to STITCHER with different sampled speed sets $\mathcal{V}_m$. 
Each set $\mathcal{V}_m$ consists of $k$ discrete speeds sampled from the interval $[0,\,v_{max}]$, such that $|\mathcal{V}_m| = k$.
In order to ensure a fair comparison, these sets must be subsets of those used by Dense Dijkstra. 
In other words, we ensure $k \leq K$, where $K$ denotes the number of speeds sampled in Dense Dijkstra, and our sampled sets are as evenly distributed as possible within the interval. 

\Cref{fig:pathcost_planning_discr} shows the relative execution time (left) and computation times (right) for different sizes of $\mathcal{V}_m$ compared to Dense Dijkstra.
As expected, the execution time increases as the sampled speed set becomes sparser because the planner has fewer speed options available.
However, the observed increase in execution time is at most 8\% even with the sparsest speed set tested ($|\mathcal{V}_m| = 5$).
Critically, the minimal increase in execution time is accompanied by a significant reduction in computation time: a speed-up of four orders of magnitude with the sparsest sampled speed set tested.
Although a sample speed set size of $|\mathcal{V}_m| = 11$ yields nearly identical execution times to the dense set while offering three orders of magnitude improvement in computation time, adequate performance is achieved when $|\mathcal{V}_m| = 5$.
Hence, we use the set
$\mathcal{V}_m = \left\{0,\, 0.25\,v_{max}, \,0.5\,v_{max}, \,0.75\,v_{max}, \, v_{max}\right\}$ for the remainder of our analysis.
A representative speed profile for different sized speed sets is shown in \cref{fig:vel_norm}.

\begin{figure}[t!]
  \centering
  \includegraphics[width=\columnwidth]{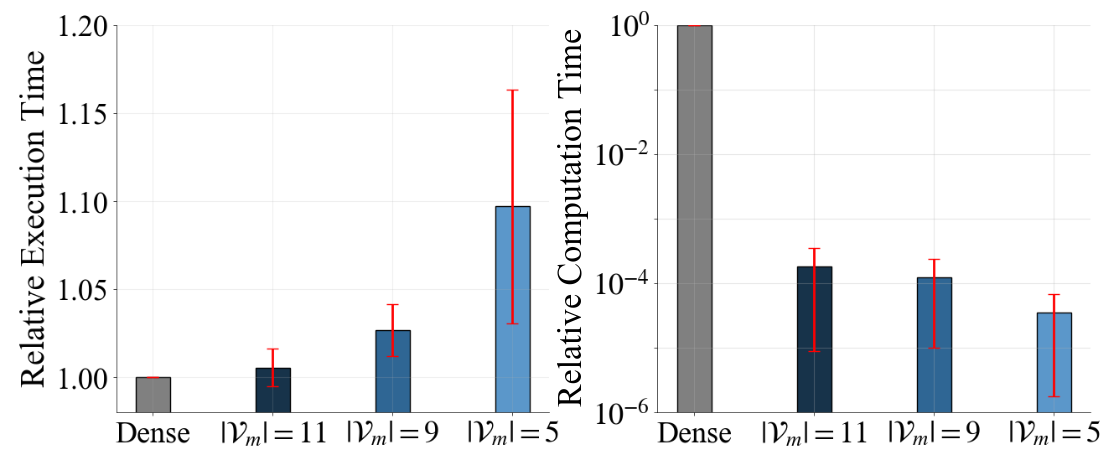}
  \caption{Analysis of speed discretization on execution and computation time. As the discretization of our method is increased, we converge to the Dense Dijkstra solution. We use $\mathcal{V}_m = 5$ for our experiments as it achieves significant computational advantages while retaining suitable performance.}
  \label{fig:pathcost_planning_discr}
\vskip -0.17in   
\end{figure}

\begin{figure}[t!]
 \centering
 \includegraphics[width=\columnwidth]{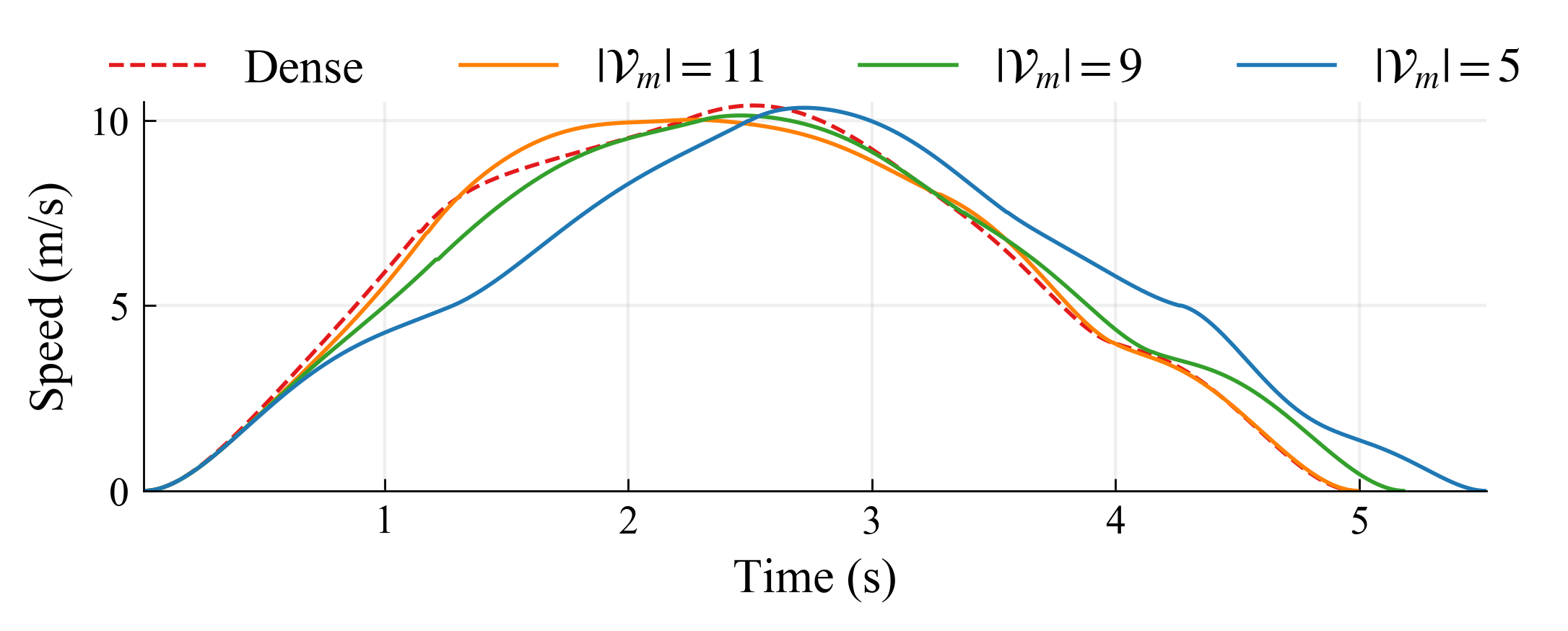}
\caption{Speed profiles of STITCHER using varied speed discretizations. As the number of speed samples increases, STITCHER converges to the optimal solution approximated by Dense Dijkstra.}
\label{fig:vel_norm}
\vskip -0.08in    
\end{figure}

\begin{figure}[t]
 \centering
 \includegraphics[width=\columnwidth]{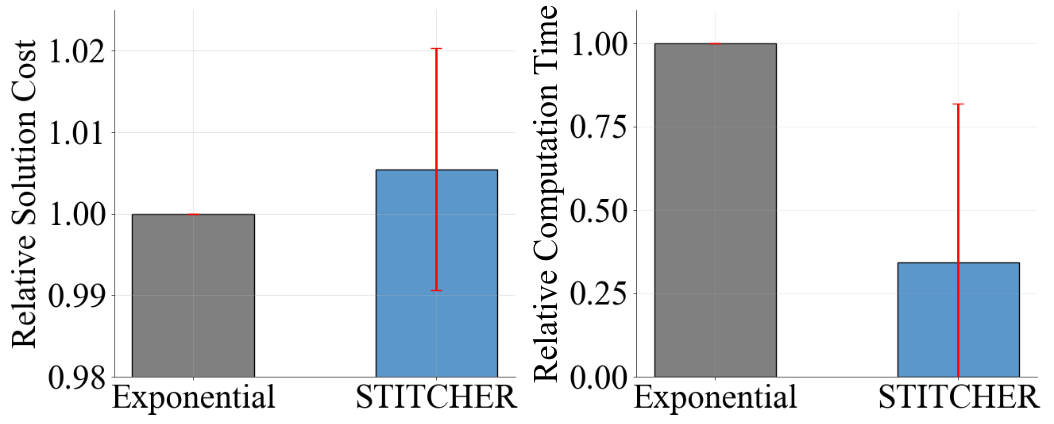}
\caption{Evaluation of the effect of a greedy pre-processing step on solution quality and computation speed. Results are from a monte carlo simulation of 50 random realizations for trajectories with between 2 and 10 waypoints.}
\label{fig:exp_v_greedy_comp_cost}
\vskip -0.1in    
\end{figure}

\subsection{Greediness Benchmarking} \label{sec:greediness_benchmarking}
STITCHER employs a greedy graph pre-processing step in order to generate a compact final search graph which enables real-time trajectory computation.
It is important to note that this greedy pre-processing step results in a graph that does not exhaustively account for all possible acceleration states at a given waypoint (see \cref{fig:greedy_graph}). 
Hence, we evaluated the extent to which this greedy pre-processing step leads to sub-optimal solutions by performing a Monte Carlo simulation comparing the path cost of STITCHER to that of an exhaustive exponentially-growing graph. 
The experiment includes 50 realizations of randomly generated initial and final positions, with the number of waypoints ranging between 2 and 10.

\Cref{fig:exp_v_greedy_comp_cost} shows the relative solution cost (left) and computation times (right) resulting from searching over the exponentially-growing graph and running STITCHER.
On average, STITCHER achieved a 0.5\% difference in solution cost, with the maximum difference being 7\%. 
Additionally, STITCHER achieves approximately 66\% speed-up in computation time as a result of the reduced graph size.
Excluding solution cases with 2 or 3 waypoints where the exponential growth of the graph including terminal acceleration states cannot be observed, STITCHER achieves an average of 82\% speed-up.
\Cref{fig:greedy_office} shows the speed profile of a case in which STITCHER achieved a solution cost that was 2\% different from that of the exponential graph. Qualitatively, the trajectories are very similar with only a minor deviation at the end of the trajectory. 
Therefore, the greedy graph pre-processing step results in a negligible increase in solution cost, but a substantial improvement in computation time. 

\begin{figure}[t!]
 \centering
 \includegraphics[width=\columnwidth]{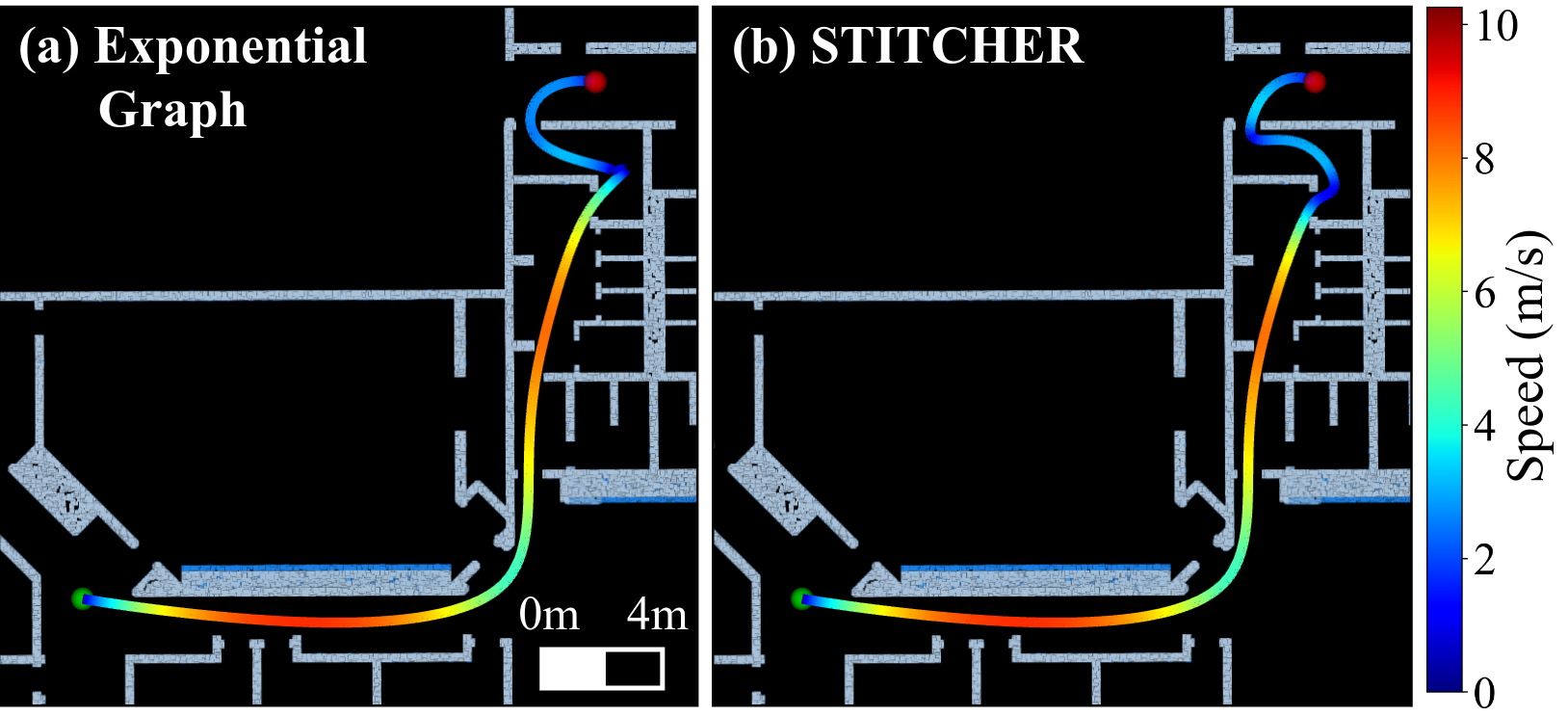}
\caption{The final speed profile given by a motion primitive search with (a) an exhaustive exponentially growing graph and (b) STITCHER.}
\label{fig:greedy_office}
\vskip -0.08in    
\end{figure}

\begingroup

\setlength{\tabcolsep}{8pt} 
\renewcommand{\arraystretch}{1.25} 
\begin{table*}[t]
\caption{Heuristic Evaluation with Varied Edge Costs and Acceleration Constraints via Percent Reduction in Generated Edges Compared to Dijkstra's.}
\vskip -0.1in  
\label{tab:heuristic_edge_explored}
\begin{center}
\begin{tabular}{c|c|c|c|c|c|c|c}
\hline
\multirow{4}{*}{\parbox{1cm}{\centering Edge Cost}} & \multirow{4}{*}{\parbox{2.5cm}{\centering Acceleration Constraint}} & \multicolumn{6}{c}{Map}\\
& & \multicolumn{3}{c|}{Perlin Noise} & \multicolumn{3}{c}{Willow Garage} \\
\cline{3-8}
& & $N=4$ & $N=6$ &  $N=8$ &  $N=4$ & $N=6$ & $N=8$\\
& & \% Red. & \% Red. & \% Red. & \% Red. & \% Red. & \% Red. \\
\hline
LQMT & Admissible Truncated Cone & 8 & 8 & 7 & 0 & 3 & 13 \\ 
Time & Admissible Truncated Cone  & 22 & 21 & 15 & 0.3 & 13 &  26 \\ 
Time & Admissible Box & 27 & 34 & 28 & 2 & 20 & 24 \\
Time & Inadmissible Truncated Cone & 45 & 44 & 22 & 9 & 38 & 39 \\
\hline
\end{tabular}
\end{center}
\vskip -0.3in  
\end{table*}

\endgroup

\subsection{Heuristic Benchmarking}\label{sec:heuristic_benchmark}
The quality of the heuristic used to guide the motion primitive search phase of STITCHER can be quantified by comparing the number of edges, i.e., motion primitives, generated by STITCHER with that of Dijkstra's algorithm.
The velocity magnitude and direction sets were kept constant across both planners with $|\mathcal{V}_m|=11$ and $|\mathcal{V}_d|= 3$.
The number of edges created is a better evaluation metric than the nodes explored because the main source of computation time comes from generating motion primitives and checking them for constraint violations. 
The effectiveness of the search heuristic depends on two main factors: how accurately it approximates the true cost-to-go, and whether the heuristic is admissible under the specified constraints. 
The edge cost used in the final search affects the tightness of the heuristic approximation, while the acceleration constraint imposed to prune motion primitives can influence admissibility. 
Therefore, we assess the effectiveness of STITCHER's heuristic with various graph edge costs and acceleration constraints to provide a general understanding of how search performance is impacted.   

\subsubsection{Varying Edge Cost}
We evaluate the quality of the heuristic with two different edge costs: (1) equivalent to the LQMT performance index and (2) execution time. 
\Cref{tab:heuristic_edge_explored} shows the percent reduction in the number of edges created for STITCHER compared to Dijkstra's algorithm.
Using the LQMT edge cost in the final search, STITCHER creates 7\% fewer edges on average compared to Dijkstra's algorithm in the Perlin Noise environment and 5\% in the Willow Garage environment.
Greater edge reduction can be achieved by defining the edge cost to be the trajectory execution time.
Recall that the search heuristic is the minimum time required to reach the goal using a double integrator model.  
Therefore, by using an edge cost soley defined by trajectory duration, the heuristic now more closely approximates the cost-to-go.
This tighter approximation is reflected in the results of \Cref{tab:heuristic_edge_explored}, where STITCHER generates an average of 20\% fewer edges in Perlin Noise and 13\% fewer in  Willow Garage.
The reduced effectiveness of the heuristic in Willow Garage was attributed to having more tight corridors, so more motion primitives were in collision (see \cref{fig:environment_comparison}).
Therefore, while environment dependent, using a search heuristic that closely approximates the remaining cost to the goal greatly reduces search effort. 

\begin{figure}[t!]
 \centering
 \includegraphics[width=\columnwidth]{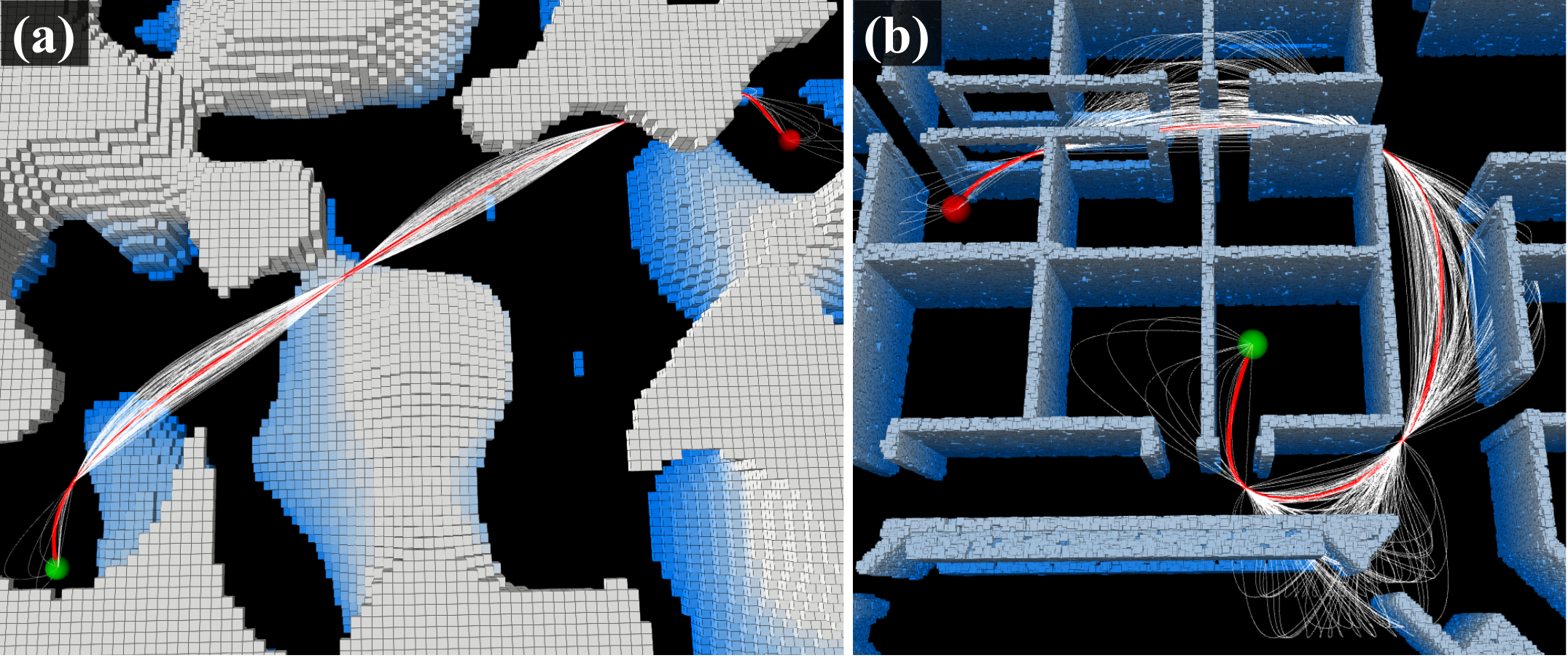}
 \caption{Final trajectories (red) through a six waypoint path in (a): Perlin Noise and (b): Willow Garage environment. The trajectory options (white) which inform the heuristic are more likely to be in collision in the Willow Garage environment due to tight corridors.}
\label{fig:environment_comparison}
\end{figure}

\subsubsection{Varying Acceleration Constraint}
Varying acceleration constraints may also lead to differing performance of the heuristic as this constraint directly affects admissibility. 
We tested three different acceleration constraints: (1) admissible truncated cone, (2) admissible box, and (3) inadmissible truncated cone (See \cref{fig:heuristic_accel_constr}). 
The truncated cone constraint comes from the resulting achievable mass-normalized thrust volume with constraints on the magnitude and tilt. 
This constraint $c({\boldsymbol{a}})$ is admissible when it is maximally bounded by the maximum control input $u_{max}$ imposed during heuristic generation in Stage 2 (see \cref{fig:heuristic_accel_constr}a), while the inadmissible variant applies a smaller $u_{max}$, e.g., $u_{x,max} \leq \max_{a_x}c({\boldsymbol{a}})$ (see \cref{fig:heuristic_accel_constr}c). 
An admissible box constraint is an acceleration constraint that perfectly matches that of Stage 2 (see \cref{fig:heuristic_accel_constr}b).
\Cref{tab:heuristic_edge_explored} shows that improved performance can be achieved by applying an admissible box constraint on acceleration compared to the admissible truncated cone.
In this case, STITCHER achieves approximately a 30\% reduction in edges in the Perlin Noise environment and a 15\% reduction in the Willow Garage compared to Dijkstra's algorithm.
Because the box constraint on acceleration is equivalent to that used in the heuristic generation, the STITCHER algorithm guides the search away from branches that are likely to exceed acceleration constraints. 
Further improvement in search speed may be observed while using an inadmissible truncated cone constraint, and is reflected in \cref{tab:heuristic_edge_explored}, where this algorithm variant achieves the highest percent reduction among all STITCHER variations. 
By allowing the final motion primitive search to have a larger allowance in achievable acceleration, the true edge cost may be smaller than that estimated by the heuristic. 
The use of inadmissible heuristics in graph search foregos desirable optimality guarantees, but can more rapidly motivate the search toward the goal as a larger weighting is placed on the heuristic cost. 
\begin{figure}[t]
 \centering
 \includegraphics[width=\columnwidth]{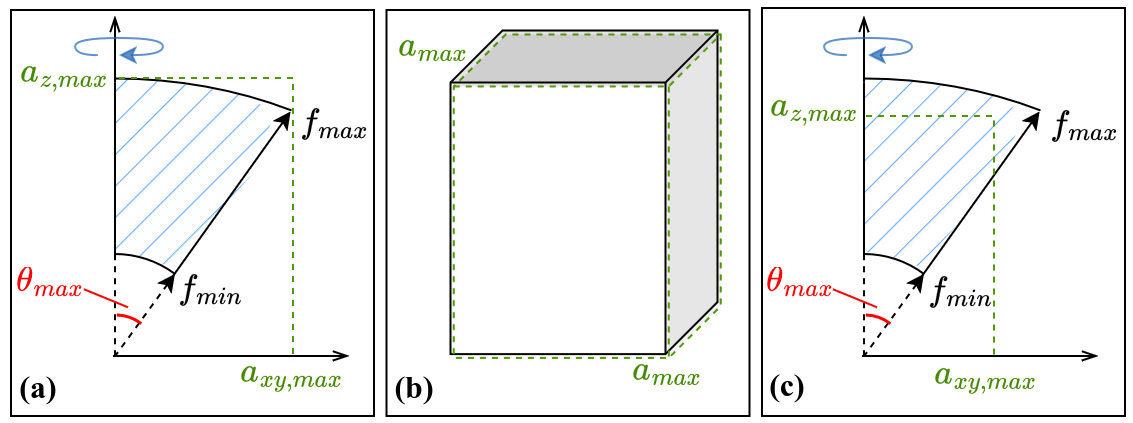}
 \caption{Different acceleration constraints used for heuristic study: (a) admissible truncated cone, (b) admissible box, and (c) inadmissible truncated cone. Green dashed outlines are Stage 2 constraints used for heuristic creation while solid black outlines denote constraints imposed for the final trajectory.}
\label{fig:heuristic_accel_constr}
\end{figure}

\subsubsection{Summary} This study showed that the proposed heuristic is effective, but its performance depends on the environment, graph edge cost, and the chosen acceleration constraint.
For the remaining experiments, we used an LQMT edge cost to ensure fair comparisons to state-of-the-art algorithms, and apply an admissible truncated cone acceleration constraint to retain graph search optimality guarantees and accurately reflect the true physical limits of quadrotor systems. 

\subsection{Computation Time Composition Analysis}
\Cref{fig:planning_time} shows the computation time of the different components of STITCHER averaged over six tested trials.
The average time to perform the velocity graph search is 1.8 ms, compared to 2.2 ms for the motion primitive search.
Although both searches have the same graph size, it is important to note that the motion primitives from \eqref{eq:triple_min_effort} are more time-consuming to generate than the minimum time trajectories used in the velocity graph. 
Therefore, the similar computation times arise from the search heuristic reducing the number of edges generated.
The low computation time of the velocity search further indicates its effectiveness in computing an informative admissible heuristic for the motion primitive search.
Finally, constraint-checking with uniform samples at every 0.1 s averaged only 0.3 ms, showing the efficacy of the relatively simple approach.
The time to build the \textit{kd}-tree is not included in \cref{fig:planning_time} as it is a one-time occurrence at simulation start-up. It takes approximately 27 ms and 83 ms to build the tree for the Perlin and Willow Garage environment respectively.

\begin{figure}[t]
 \centering
 \includegraphics[width=\columnwidth]{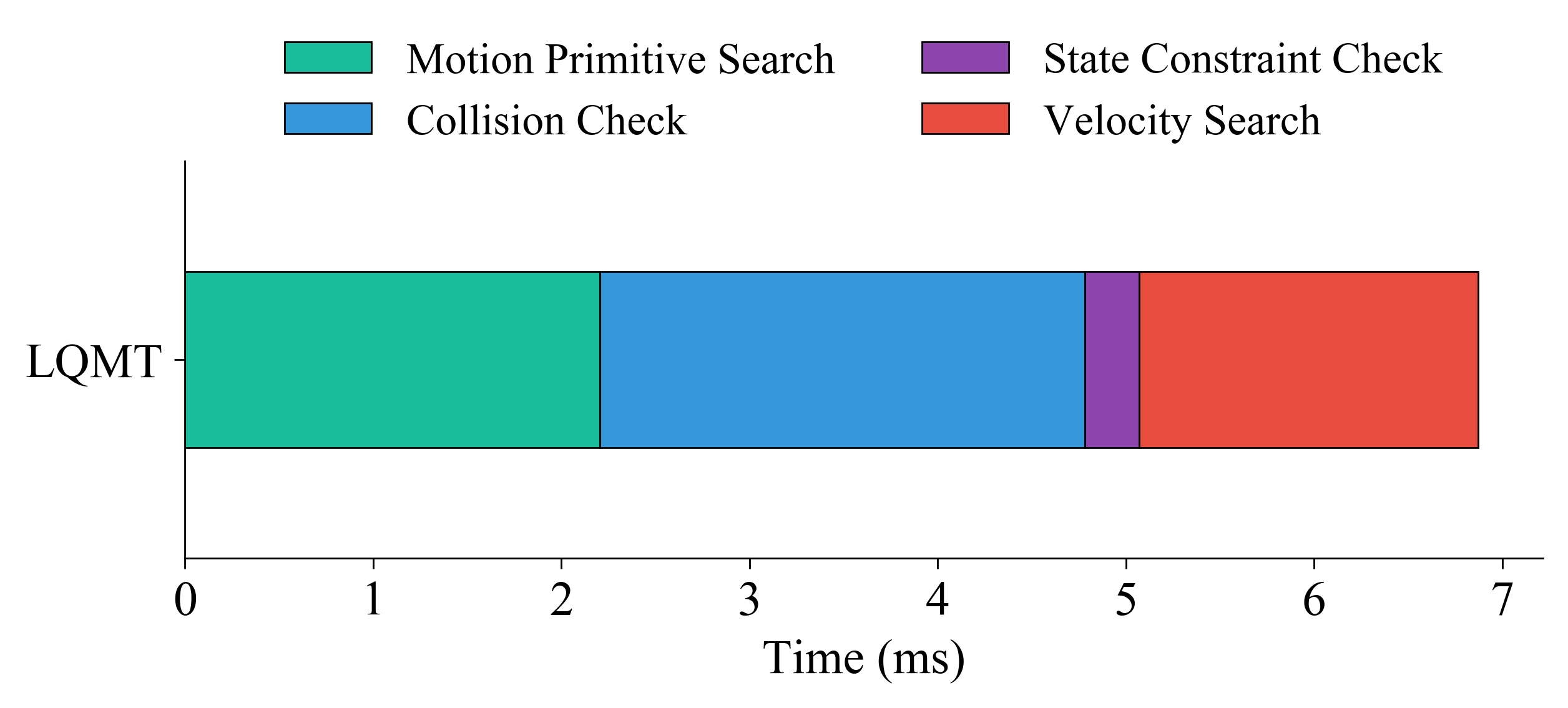}
 \caption{The average contribution of different path planning components.}
\label{fig:planning_time}
\end{figure}

\subsection{Constrained Motor Forces} \label{sec:indiv_motor}
We present a case study demonstrating STITCHER's ability to generate trajectories that satisfy individual thruster limits for a quadrotor, which also extends to lander vehicles. 
We assume the vehicle dynamics takes the form

\begin{equation*}
    \begin{aligned}
        m \ddot{\boldsymbol{r}} & = R\,  \boldsymbol{T} - m\boldsymbol{g} \\
        \dot{R} & = R \left[\boldsymbol{\omega}\right]_\times \\
        J \dot{\boldsymbol{\omega}} & = \boldsymbol{\tau} - \boldsymbol{\omega}\times J \boldsymbol{\omega},
    \end{aligned}
\end{equation*}
where $\boldsymbol{r}$ is the vehicle's position vector, $R$ is the rotation matrix that represents the vehicle orientation with respect to an inertial frame, $\boldsymbol{T} = T \hat{\boldsymbol{e}}_z$ is the thrust vector directed along the body $z$-axis, $\boldsymbol{g}$ is the gravity vector,  $\boldsymbol{\omega}$ is the angular velocity vector, $[\cdot]_\times$ is the cross product matrix, $J$ is the inertia tensor, and $\boldsymbol{\tau}$ is the body torque vector. 

For the case study, we assume a quadrotor with a mass $m = 0.8 \text{ kg}$, moment of inertia $J = \text{diag}(0.05,0.05,0.05) \text{ kg m}^2$, arm length $l = 12.5 \text{ cm}$, motor drag coefficient $c = 0.2$ m, minimum single-motor thrust $F_{min} = 0.16 \text{ N}$, and maximum single-motor thrust $F_{max} = 3.75 \text{ N}$. Each motor thrust $F_i$ is constrained by 
\begin{align*}
0 < F_{min} \leq F_i \leq F_{max}, \quad \forall i \in \{1,2,3,4\}.
\end{align*}
At any given time, individual motor thrusts $F_i$ are found by solving a linear system that depends on motor placement and physical vehicle properties, i.e., solving
\begin{align*}
\begin{bmatrix}
T\\
\tau_x \\
\tau_y \\ 
\tau_z 
\end{bmatrix} =
M
\begin{bmatrix}
F_1 \\
F_2 \\
F_3 \\
F_4
\end{bmatrix},
\end{align*}
for an invertible allocation matrix $M \in \mathbb{R}^{4\times4}$ with $T$ denoting the magnitude of the thrust $T = m\|\boldsymbol{f}\|_2$ and $\boldsymbol{f}$ if the mass normalized thrust vector. 
The torque $\boldsymbol{\tau} = [\tau_x$, $\tau_y$, $\tau_z]^\top$ along the trajectory is $\boldsymbol{\tau} = J\dot{\boldsymbol{\omega}} + \boldsymbol{\omega}\times J \boldsymbol{\omega}$ and  can be found via differential flatness, with components of $\boldsymbol{\omega}$ and $\dot{\boldsymbol{\omega}}$ given by
\begin{align*}
\begin{bmatrix}
{\omega}_y \\
-{\omega}_x \\
0
\end{bmatrix} = R^\top \dot{\hat{\boldsymbol{T}}},
~~
\begin{bmatrix}
\dot{\omega}_y \\
-\dot{\omega}_x \\
0
\end{bmatrix} = R^\top \ddot{\hat{\boldsymbol{T}}} - \boldsymbol{\omega} \times R^\top \dot{\hat{\boldsymbol{T}}},
\end{align*}
where $\hat{\boldsymbol{T}}$ is the normalized thrust vector.

\begin{figure*}[t]
\begin{subfigure}{\textwidth}
\centering
 \includegraphics[width=\columnwidth]{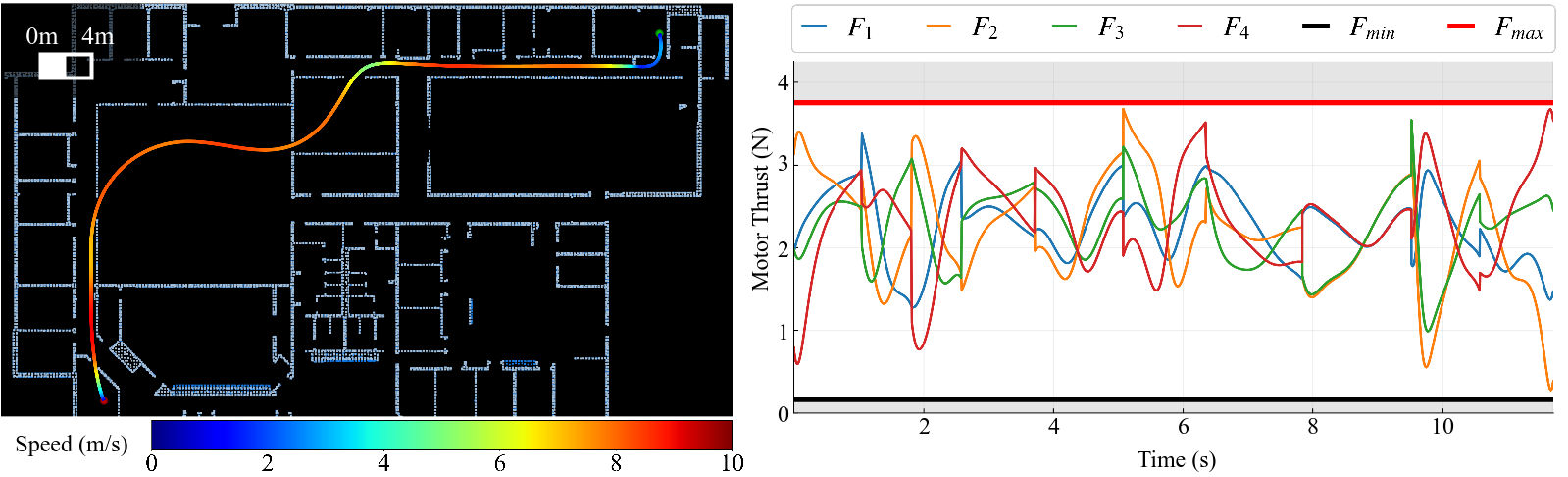}
 \caption{Constrained Individual Motor Commands.}
 \label{fig:motor_command_enforced}
 \vspace{6pt}
\end{subfigure}
\\
\begin{subfigure}{\textwidth}
\centering
 \includegraphics[width=\columnwidth]{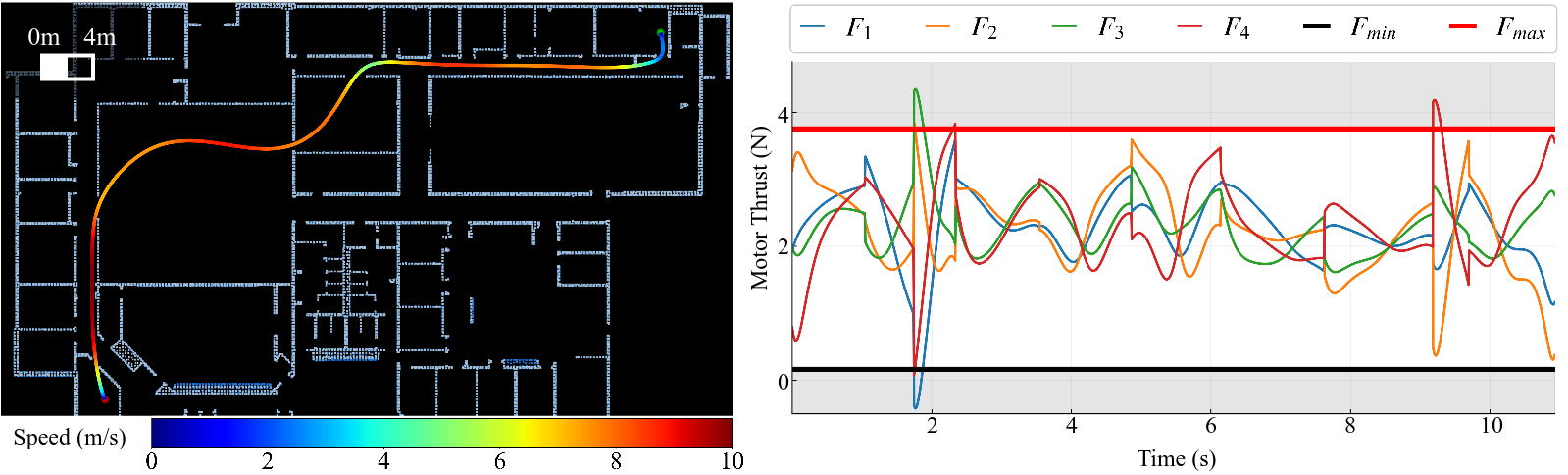}
 \caption{Unconstrained Individual Motor Commands.}
 \label{fig:motor_command_not_enforced}
\end{subfigure}
\caption{Experiment in the Willow Garage environment comparing trajectory plans that (a): constrain individual motors by the maximum thrust limits given by manufacturer specifications and (b): do not enforce individual motor constraints. Left: Trajectory colored by speed. Right: The corresponding thrust profile of individual motors depicting constraint satisfaction.}
\label{fig:motor_command_traj}
\vspace{-0.1in}
\end{figure*}

We conducted an experiment in the Willow Garage environment, where the tight corridors may make individual motor constraint satisfaction more challenging.
To further increase problem difficulty, we use a quadcopter model with a large moment of inertia, resulting in large thrust commands for rotational maneuvers.
In \cref{fig:motor_command_traj}, we use STITCHER to plan two trajectories: one in which thruster limits are enforced (\cref{fig:motor_command_enforced}) and one in which they are unenforced (\cref{fig:motor_command_not_enforced}). 
It is important to note that the unconstrained case still imposes a limit on the overall thrust (sum of individual thrusters) but does not enforce the constraint on a per-motor basis. 
As shown in \cref{fig:motor_command_enforced}, the constrained trajectory still achieves high speeds while strictly obeying the per-motor limits on thrust.
On the other hand, \cref{fig:motor_command_not_enforced} shows that the unconstrained trajectory violates these limits for several motors.
In one instance, $F_1$ requires a motor to provide a negative thrust which is physically impossible and would certainly lead to catastrophic failure as motor performance is already unreliable at low throttle.
These results showcase the versatility of our framework in adhering to highly nonconvex constraints that are directly limited by hardware.

\subsection{Comparison with State-of-the-Art}
We compared STITCHER with three state-of-the-art algorithms: GCOPTER \cite{Wang22:Geometrically}, FASTER \cite{Tordesillas22:FASTER}, and MIGHTY \cite{Kondo26:Mighty}.
These methods were selected for comparison because they similarly consider a known environment setting (i.e., plan on an infinite-horizon) \cite{Wang22:Geometrically} or can be faithfully converted to the infinite-horizon setting 
\cite{Tordesillas22:FASTER, Kondo26:Mighty}.
The motion primitive search planner \cite{Liu18:Search_Based} and mixed-integer optimization method \cite{Marcucci23:Motion} are omitted from the comparison, as their computation time exceeds several seconds. We note that there are many other novel receding-horizon planners \cite{Romero22:Model, Krinner24:MPCC++, Ren22:Bubble, Lu25:hepp} but conversion to infinite-horizon makes the comparison inequitable.

Having established a common planning setting, we briefly summarize the selected baseline algorithms used for comparison. 
FASTER solves a mixed integer quadratic program, with integer variables assigning polynomial segments to different safe flight corridors. 
GCOPTER performs an online optimization over a polynomial trajectory by incorporating state constraints into the cost functional and running a quasi-newton method.
Similarly, MIGHTY runs a quasi-newton optimization but parametrizes the trajectory as a Hermite spline. 
All algorithms rely on a sparse geometric path to form safe flight corridors, but do not enforce final trajectories to pass through waypoints.
We evaluate each planner by time (computation time versus execution time) and
failure (constraint violation or incomplete/no path found).
For the Perlin Noise environment, the path lengths were 12.5 m, 30 m, and 55 m while the path lengths for Willow Garage were 20 m, 25 m, and 30 m with $N =$ 4, 6, 8 waypoints for both environments.

\begin{table*}[t]
\caption{State-of-the-Art Comparison Time Analysis.}
\vskip -0.1in  
\label{tab:lqmt_SOA_comp}
\begin{center}
\begin{tabular}{c|c|c|c|c|c|c|c|c|c}
\hline
\multirow{2}{*}{Map}  & \multirow{2}{*}{$N$} & \multicolumn{4}{c|}{Computation time (ms)} &\multicolumn{4}{c}{Execution time (s)}\\
& & \cite{Tordesillas22:FASTER} & \cite{Wang22:Geometrically} & \cite{Kondo26:Mighty} & Ours &  \cite{Tordesillas22:FASTER} & \cite{Wang22:Geometrically} & \cite{Kondo26:Mighty} &  Ours  \\
\hline
\multirow{3}{*}{\parbox{0.9cm}{\centering Perlin Noise}} & 4 & 109 & 19.0 & 56.3 & \textbf{4.03} & 3.43 & 3.51 & 4.20 & 3.67\\
& 6  & 192 & 59.4 & 54.7 & \textbf{8.30} & 4.42 & 4.42 & 5.18 & 5.75\\
& 8 & 1020 & 116 & 78.7 & \textbf{14.0} & 7.62 & 7.34 & 9.02 & 9.62 \\
\hline
\multirow{3}{*}{\parbox{0.9cm}{\centering Willow Garage}} & 4 & 240 & 44.1 & 33.8 & \textbf{6.04} & 4.40 & 4.37 & 5.97 & 4.53 \\
& 6 & 5240 & 59.9 & 41.9 & \textbf{11.7} & 8.25 & 5.97 & FAILED & 6.30 \\
& 8 & 10400 & 110 & 103 & \textbf{14.3} & 10.5 & FAILED & FAILED & 7.95\\
\hline
\end{tabular}
\end{center}
\vskip -0.15in  

\end{table*}

\subsubsection{Time Analysis}
\Cref{tab:lqmt_SOA_comp} compares the computation times and trajectory execution times of each planner given the six different geometric paths.
STITCHER's computation times are faster than those measured for FASTER, GCOPTER, and MIGHTY for every test, and are up to 7x faster than MIGHTY, 8x faster than GCOPTER and 700x faster than FASTER. 
In some cases the baseline planners achieved lower execution times.
This can be attributed to their ability to treat waypoints as soft constraints, i.e., the trajectory is only required to pass nearby a waypoint rather than through it, as well as the chosen resolution of state samples in STITCHER. 
The difference is most noticeable in the Perlin Noise environment where more free space is available for waypoint adjustment. 
\Cref{tab:sampling_wp_ablation} compares the planning and execution times of GCOPTER \cite{Wang22:Geometrically} to different algorithm variants of STITCHER evaluating the effect of sampling and waypoint selection in the Perlin Noise environment. 
Case S increases the sampled speed set to $|\mathcal{V}_m|=9$, Case W utilizes the optimized waypoints from GCOPTER, and Case SW is a combination of both strategies. All case variants reduce STITCHER's baseline execution time, motivating future work in incorporating waypoint flexibility.

\begin{table}[t]
\caption{Effect of Different Waypoints and Sampling in Perlin Noise.}
\vskip -0.1in  
\label{tab:sampling_wp_ablation}
\begin{center}
\begin{tabular}{c|c|c|c|c|c|c|c|c}
\hline
\multirow{3}{*}{$N$} & \multicolumn{4}{c|}{Computation} time (ms) &\multicolumn{4}{c}{Execution time (s)}\\
& & \multicolumn{3}{c|}{STITCHER} & &  \multicolumn{3}{c}{STITCHER} \\
& \cite{Wang22:Geometrically} & S & W & SW&  \cite{Wang22:Geometrically} & S & W & SW \\
\hline
4 & 27.2 & 14.8 & 3.01 & 10.5 & 3.51 & 3.44 & 3.61 & 3.56\\
6 & 72.0 & 44.7 & 13.8 & 42.4 & 4.42 & 5.43 & 5.38 & 4.90\\
8 & 130 & 64.1 & 24.6 & 86.9 & 7.34 & 9.13 & 8.99 & 9.26 \\
\hline
\end{tabular}
\end{center}
\footnotesize{\textbf{Abbreviations}: S -- speed sample set expanded to $|\mathcal{V}_m| = 9$; W -- optimized waypoints from \cite{Wang22:Geometrically} used; SW -- combination of S and W.}
\end{table}

\begin{table*}[t!]
\caption{State-of-the-Art Comparison Failure Analysis.}
\vskip -0.1in  
\label{tab:SOA_failure}
\begin{center}
\begin{tabular}{c|c|c|c|c|c|c|c|c|c|c|c|c}
\hline
\multirow{2}{*}{Map} & \multicolumn{4}{c|}{No Path Found (\%)}  & \multicolumn{4}{c|}{State Violation (\%)} & \multicolumn{4}{c}{Collisions (\%)} \\
& \cite{Tordesillas22:FASTER} & \cite{Wang22:Geometrically} & \cite{Kondo26:Mighty} &  Ours &  \cite{Tordesillas22:FASTER} & \cite{Wang22:Geometrically} & \cite{Kondo26:Mighty} &   Ours &
\cite{Tordesillas22:FASTER} & \cite{Wang22:Geometrically} & \cite{Kondo26:Mighty} &  Ours \\
\hline
\parbox{0.82cm}{\vspace{0.08cm} \centering Perlin Noise\vspace{0.08cm}} & 6 & \textbf{0} & 2 & \textbf{0} & 2 & 4 & \textbf{0} & \textbf{0} & 6  & 6 & 26 &\textbf{0} \\
\hline
\parbox{0.82cm}{\vspace{0.08cm}\centering Willow Garage\vspace{0.08cm}} & 14 & \textbf{0} & 2 & 2 & 6 & \textbf{0} & \textbf{0} & \textbf{0} & 2 & 20 & 26 &\textbf{0} \\
\hline
\end{tabular}
\end{center}
\vskip -0.2in 
\end{table*}

\subsubsection{Failure Analysis}
A Monte Carlo simulation composed of 50 realizations was conducted to evaluate the different modes of failure experienced by each planner. \Cref{tab:SOA_failure} compares the rate at which each planner does not find a path, generates a trajectory violating state constraints or generates a trajectory in collision. 
The ``No Path Found" metric includes a numerical solver not returning a solution, or if the solution does not reach the goal. 
For the baseline planners,
this can occur due to numerical instability or when a feasible solution is not within the calculated set of safe corridors. 
Conversely, in STITCHER's graph search framework, a graph disconnection occurs when a feasible solution is not within the discrete set of sampled states. 
Across all test cases, STITCHER is the only planner to achieve both zero state violations and collisions. 
The single failure point was a graph disconnection in the Willow Garage environment, where narrow corridors make collisions more likely.
In the Willow Garage environment, the number of failed solutions by FASTER and collisions by GCOPTER significantly increases, while MIGHTY collides at the same rate across both environments. 
This result highlights a weakness of soft constraint penalties as safety can be exchanged for other objective costs.
In contrast, STITCHER never violates constraints (state, control, or obstacles) because all constraints are strictly enforced.
As an example, \cref{fig:mass_norm_thrust_ex} is a representative mass-normalized thrust profile generated by STITCHER which remains within the valid limits.

\begin{figure}[t]
 \centering
 \includegraphics[width=\columnwidth]{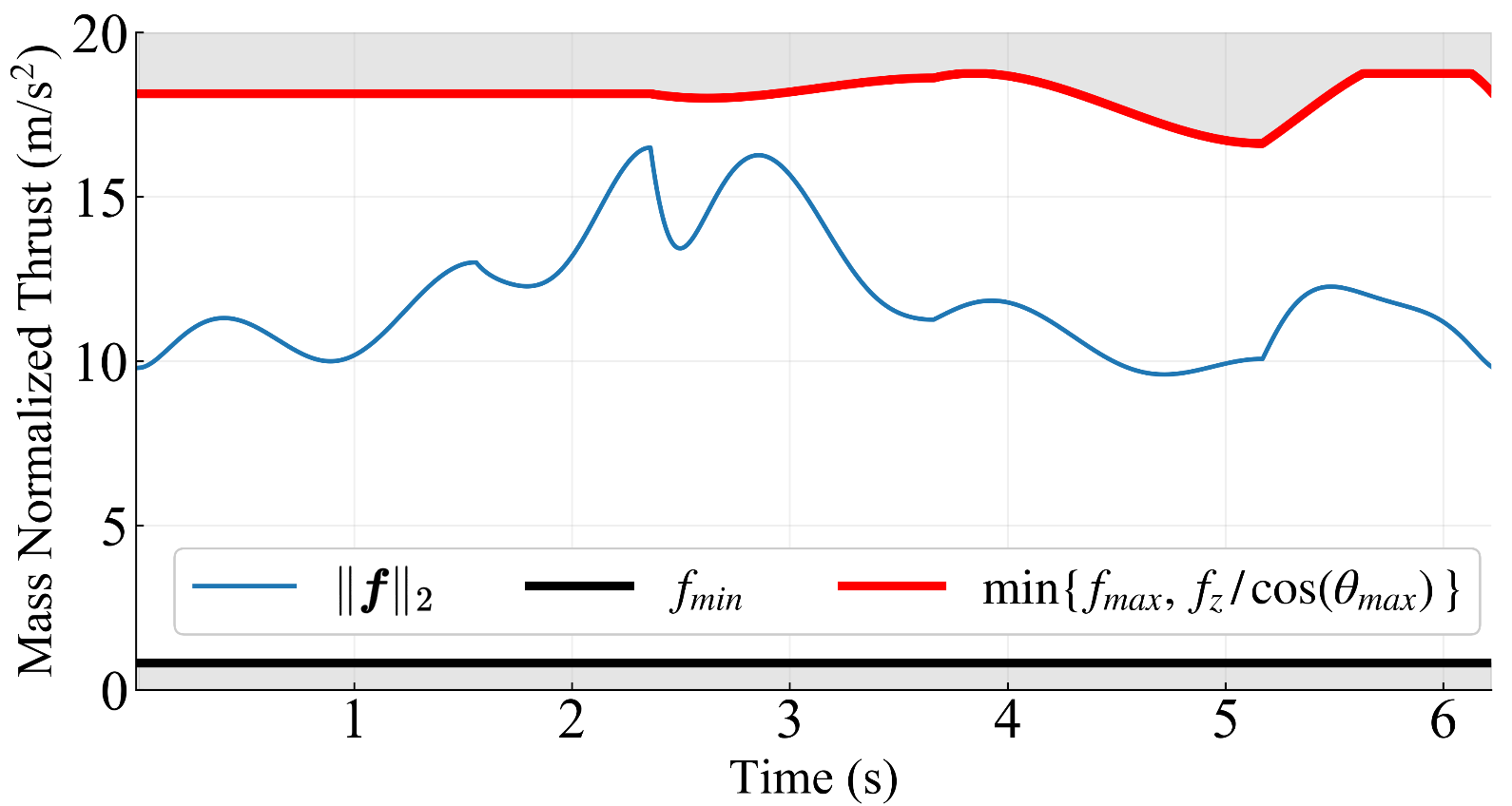}
 \caption{Example of mass-normalized thrust profile strictly satisfying constraints in Willow Garage environment.}
\label{fig:mass_norm_thrust_ex}
\end{figure}

\section{HARDWARE RESULTS} \label{sec:hardware_results}

\begin{figure}[t]
 \centering
 \includegraphics[width=\columnwidth]{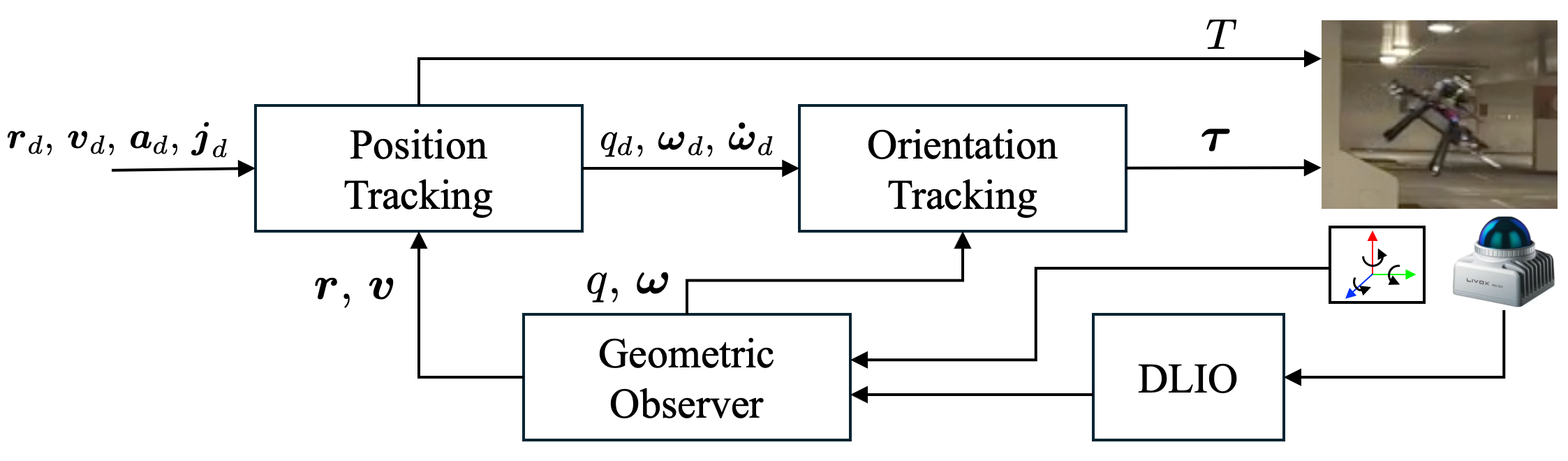}
 \caption{Control and estimation architecture for custom quadrotor platform used in hardware experiments. The vehicle is equipped with a 10th generation Intel NUC computer, Teensy 4.0 microcontroller, an InvenSense ICM-20948 IMU, and a Livox Mid-360 LiDAR.}
\label{fig:UAV_hardware}
\end{figure}

The custom quadrotor used in the experiments and its system architecture are shown in \cref{fig:UAV_hardware}.
Onboard hardware includes a 10th generation Intel NUC with i7 CPU, a Teensy 4.0 microcontroller, an InvenSense ICM-20948 IMU, a 6S 5000mAh battery, four EMAX ECO II Series 2807 1300 KV brushless motors, four 7-inch diameter tri-blade polycarbonate propellers, and a Livox Mid-360 LiDAR.
The system was flown in an underground parking structure (\cref{fig:hardware_garage_flight}) covering a flight area of approximately 45m x 35m x 3m.

\subsection{Control and Estimation Architecture}
The control and state estimation architecture used in the hardware experiments is shown in \cref{fig:UAV_hardware}.
We used the quaternion-based cascaded geometric controller \cite{Lopez25:New}, which extends \cite{Frazzoli00:Trajectory,Lee10:Geometric}, to track the trajectory generated by STITCHER with the desired acceleration and jerk as feedforward. 
The acceleration and jerk feedback needed by the cascaded geometric controller was computed by numerically differentiating the velocity feedback signal and applying a low-pass filter.
The outer loop PID position tracking controller ran on the  NUC flight computer, and the inner loop quaternion-based attitude tracking controller from \cite{Lopez17:Sliding} ran on the Teensy microcontroller; the two processors communicate via UART. 
The outer loop control rate was 100Hz while the inner loop control rate was 500Hz. 
For estimation, the Direct LiDAR-Inertial Odometry (DLIO) algorithm \cite{Chen23:DLIO} ran on the NUC computer and provided pose estimates at 100 Hz using LiDAR and inertial data.
DLIO pose estimates were fused with an IMU to estimate position, velocity, orientation, angular velocity, and IMU biases with a nonlinear geometric observer \cite{Lopez23:Contracting}.
Full state estimates were generated by the observer at 100Hz and used by the outer loop controller.

\begin{figure}[t]
 \centering
 \includegraphics[width=\columnwidth]{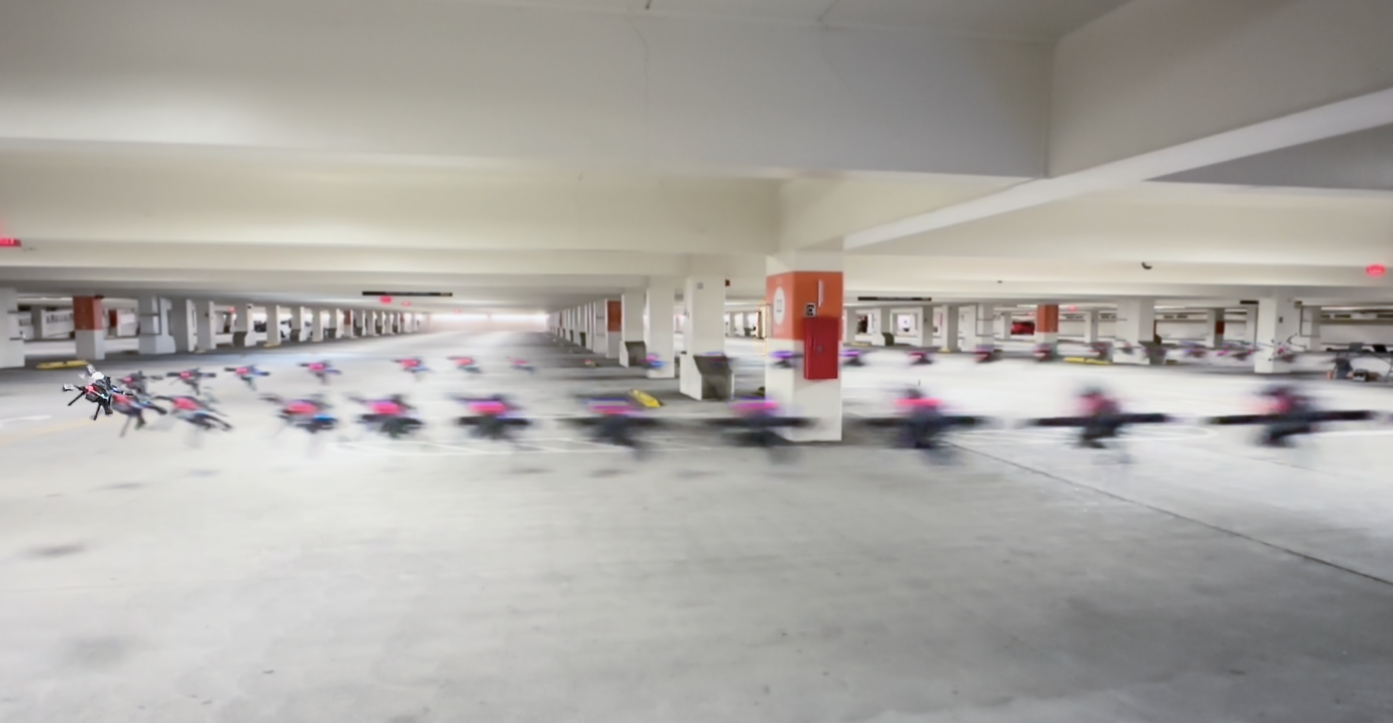}
 \caption{Motion trail of quadrotor executing trajectory from Experiment 2 in underground parking structure.}
\label{fig:hardware_garage_flight}
\end{figure}

\begin{figure*}[t]
\begin{subfigure}{0.49\linewidth}
\centering
 \includegraphics[width=\linewidth]{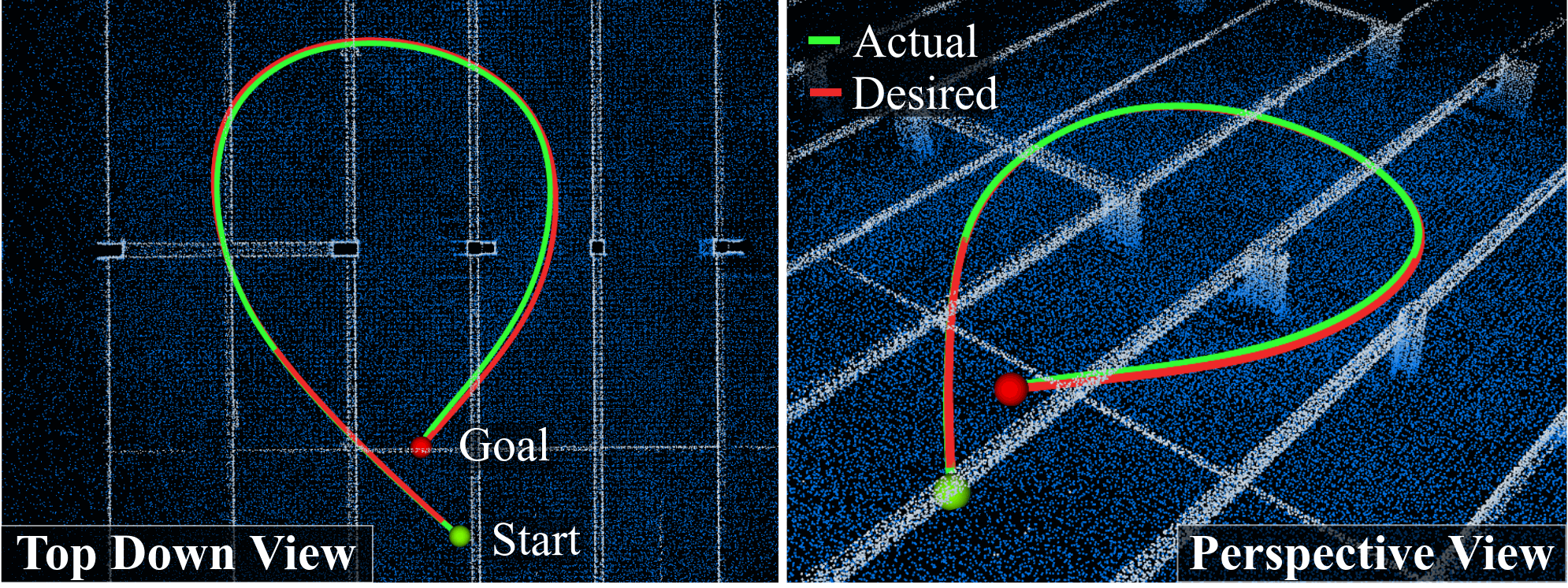}
 \caption{Experiment 1: Single-loop trajectory,  $v_{max} = 10 $ m/s}
\label{fig:hardware_p7_vmax_10}
 \vspace{6pt}
\end{subfigure}
\hfill
\begin{subfigure}{0.49\linewidth}
\centering
 \includegraphics[width=\linewidth]{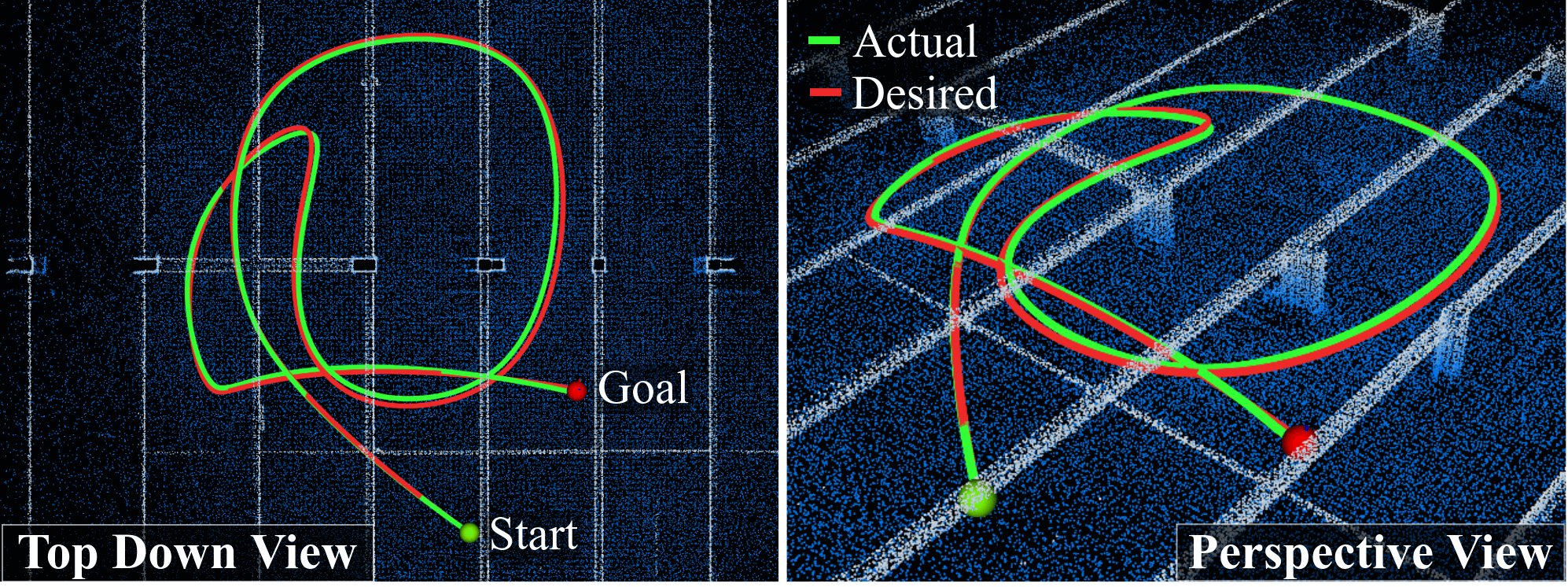}
 \caption{Experiment 2:  Multi-loop trajectory,  $v_{max} = 10 $ m/s}
\label{fig:hardware_p7_loops_vmax_10}
\vspace{6pt}
\end{subfigure}
\\
\begin{subfigure}{0.49\linewidth}
\centering
\includegraphics[width=\linewidth]{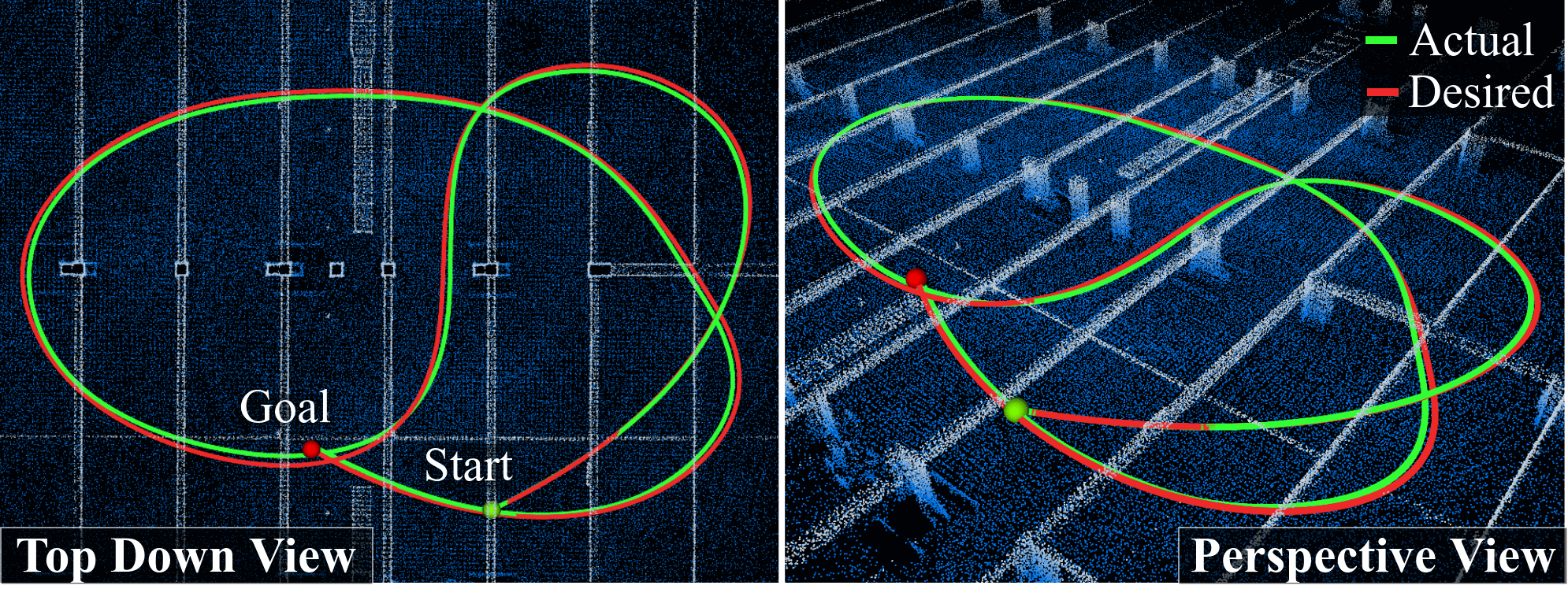}
 \caption{Experiment 3:  Multi-loop trajectory,  $v_{max} = 12.5 $ m/s}
\label{fig:hardware_p7_2_vmax_12p5}
\vspace{6pt}
\end{subfigure}
\hfill
\begin{subfigure}{0.49\linewidth}
\centering
 \includegraphics[width=\linewidth]{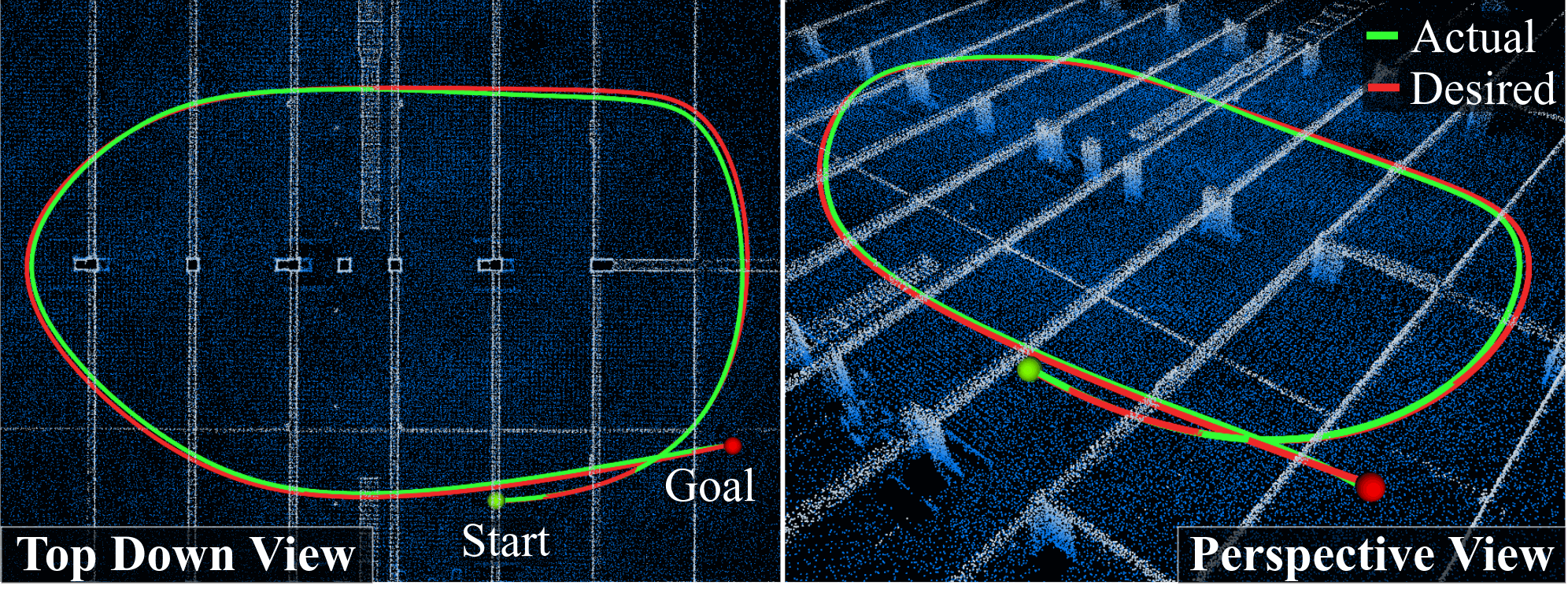}
 \caption{Experiment 4: Single-loop trajectory,  $v_{max} = 15 $ m/s}
\label{fig:hardware_p7_2_vmax_15}
\vspace{6pt}
\end{subfigure}
\\
\begin{subfigure}{0.49\linewidth}
\centering
\includegraphics[width=\linewidth]{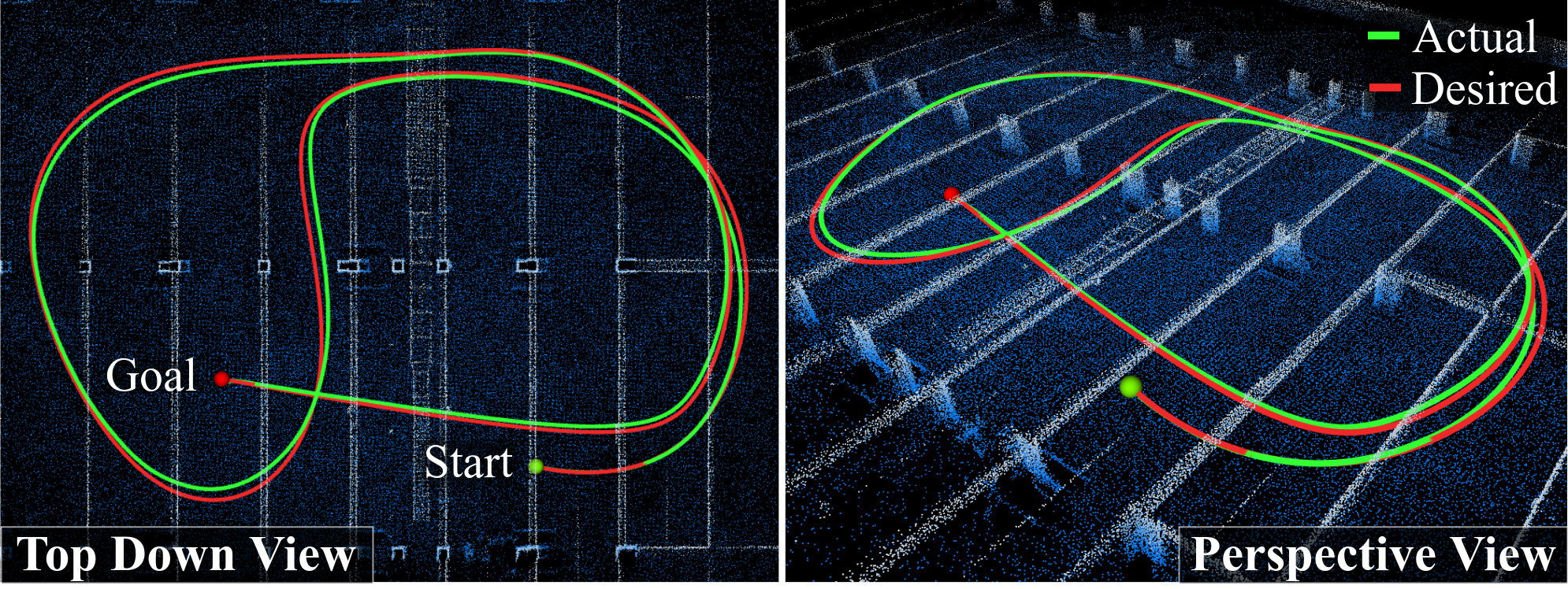}
 \caption{Experiment 5:  Multi-loop trajectory,  $v_{max} = 15 $ m/s}
\label{fig:hardware_p7_3_vmax_15_loops}
\end{subfigure}
\hfill
\begin{subfigure}{0.49\linewidth}
\centering
\includegraphics[width=\linewidth]{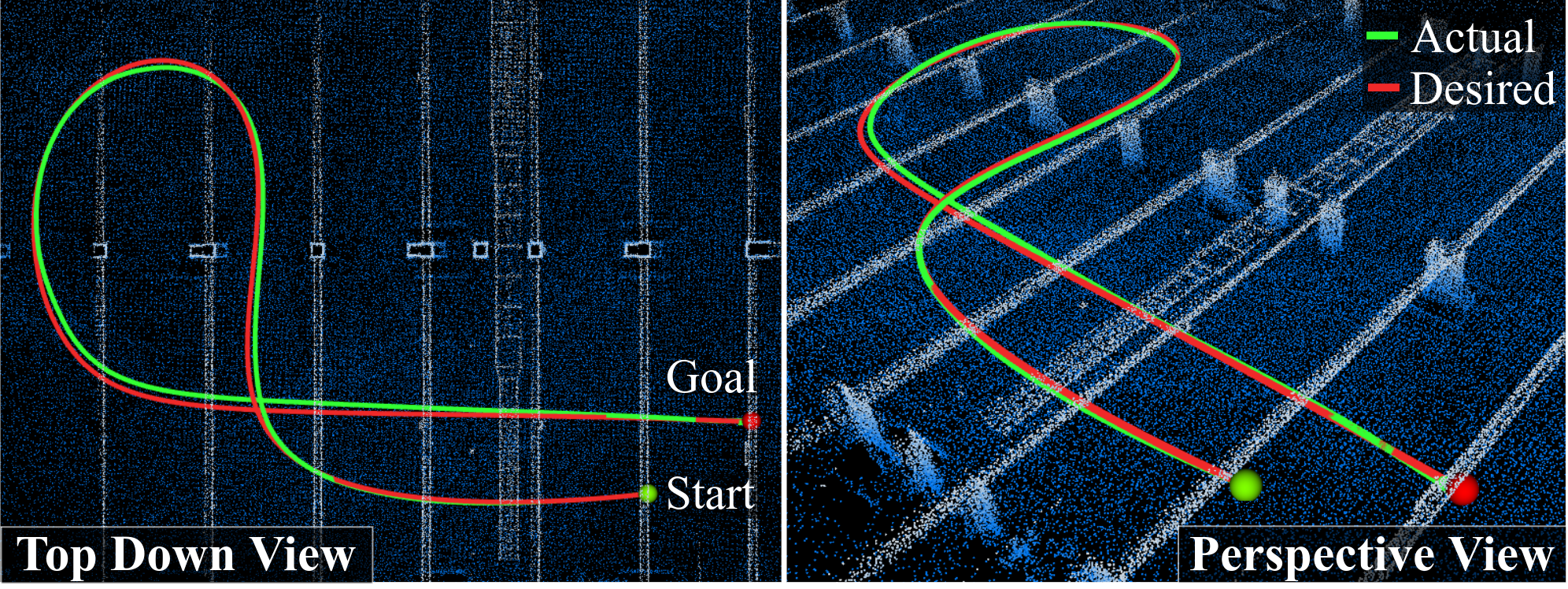}
 \caption{Experiment 6:  Single-loop trajectory,  $v_{max} = 17.5 $ m/s}
\label{fig:hardware_p7_3_vmax_17p5}
\end{subfigure}
\caption{Hardware experiments in 
Parking Structure 7 with six different trajectories: (a) Single-loop $v_{max} = 10 $ m/s, (b) Multi-loop $v_{max} = 10 $ m/s, (c) Multi-loop $v_{max} = 12.5 $ m/s, (d) Single-loop $v_{max} = 15 $ m/s, (e) Multi-loop  $v_{max} = 15 $ m/s, and (f) Single-loop $v_{max} = 17.5 $ m/s. The desired trajectory planned by STITCHER is shown in red and the actual trajectory flown by the quadcopter is shown in green. Two views of the environment are provided in each subfigure. Left: Top-down view. Right: Perspective view.}
\label{fig:rviz_hardware}
\vspace{-0.1in}
\end{figure*}

\begin{figure*}[t]
\centering
\begin{subfigure}{0.33\linewidth}
    \centering
    \includegraphics[width=\linewidth]{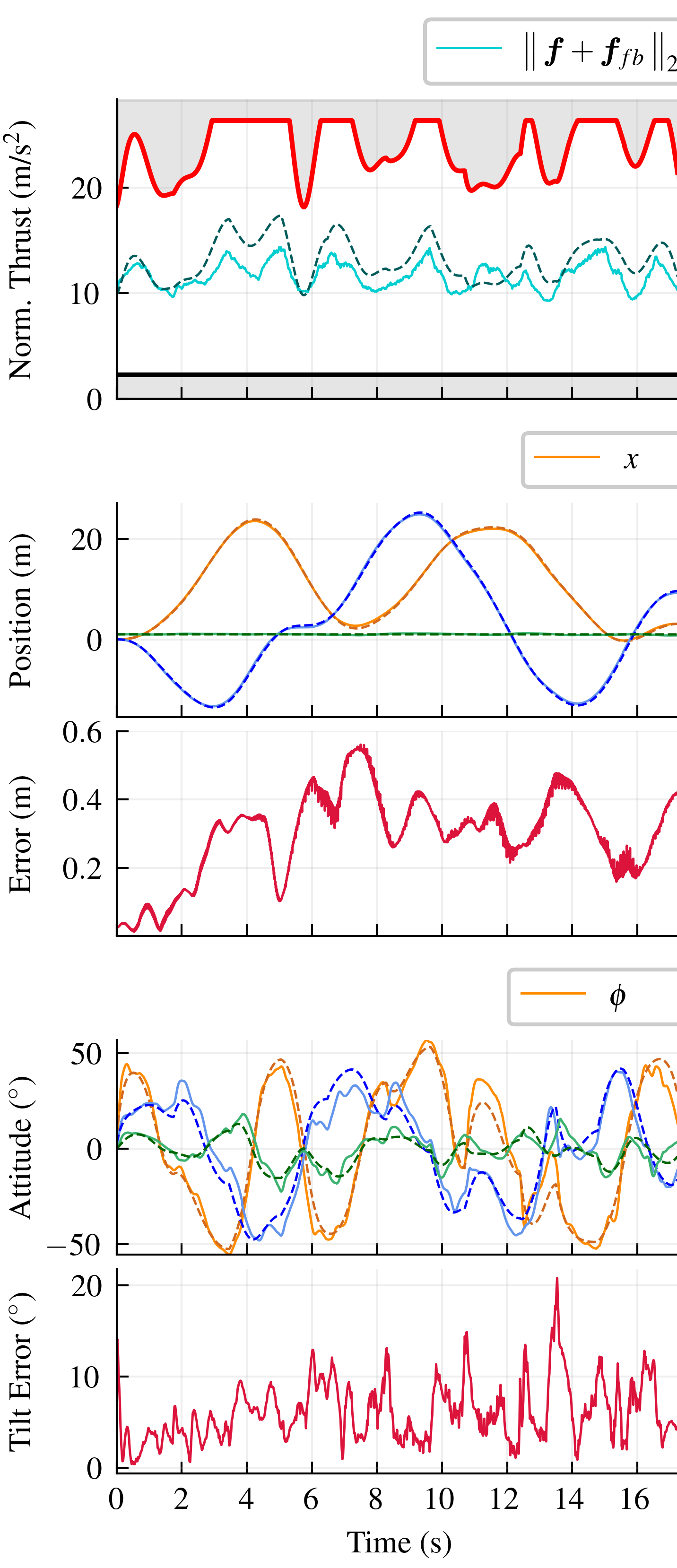}
    \caption{Experiment 3: State profiles, $v_{max} = $ 12.5~m/s.}
    \label{fig:hardware_exp3_plot}
\end{subfigure}%
\begin{subfigure}{0.33\linewidth}
    \centering
    \includegraphics[width=\linewidth]{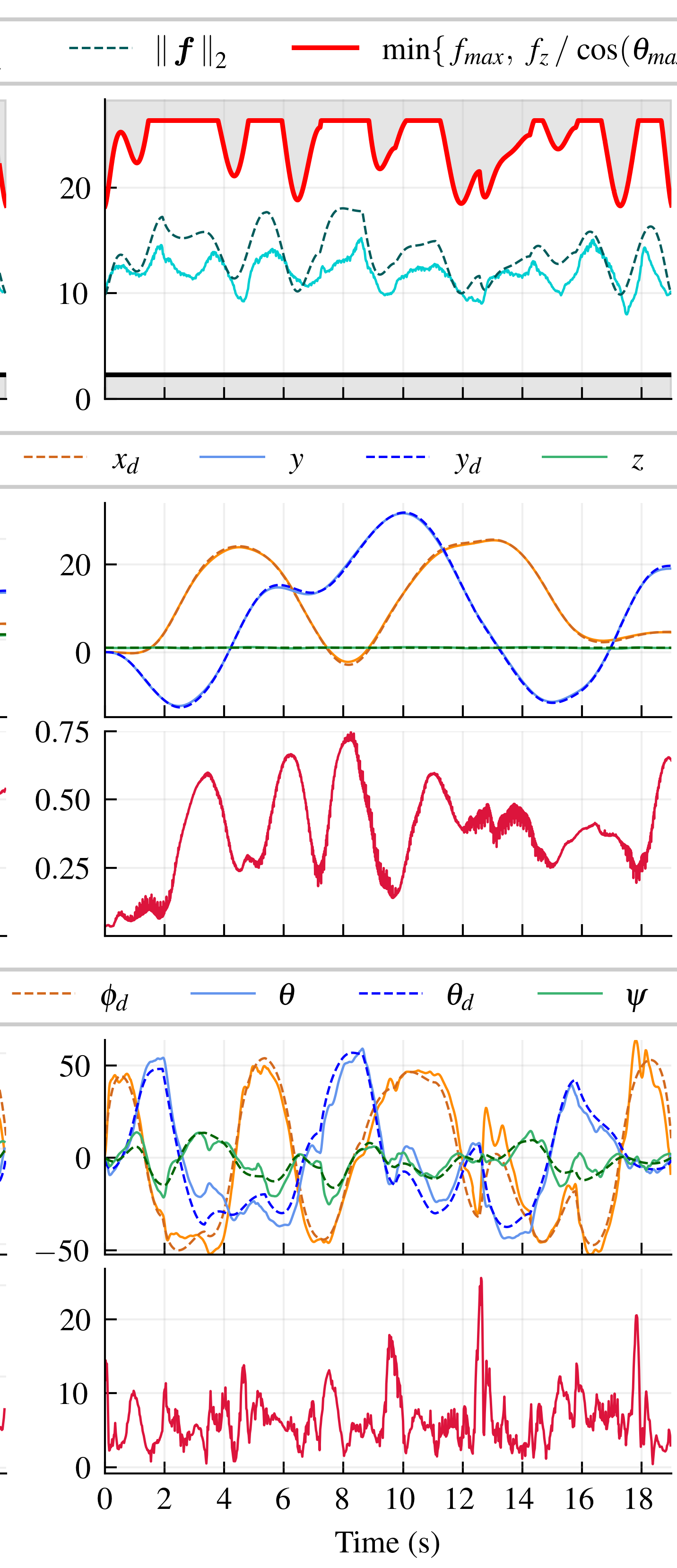}
    \caption{Experiment 5: State profiles, $v_{max} = $ 15~m/s.}
    \label{fig:hardware_exp5_plot}
\end{subfigure}%
\begin{subfigure}{0.33\linewidth}
    \centering
    \includegraphics[width=\linewidth]{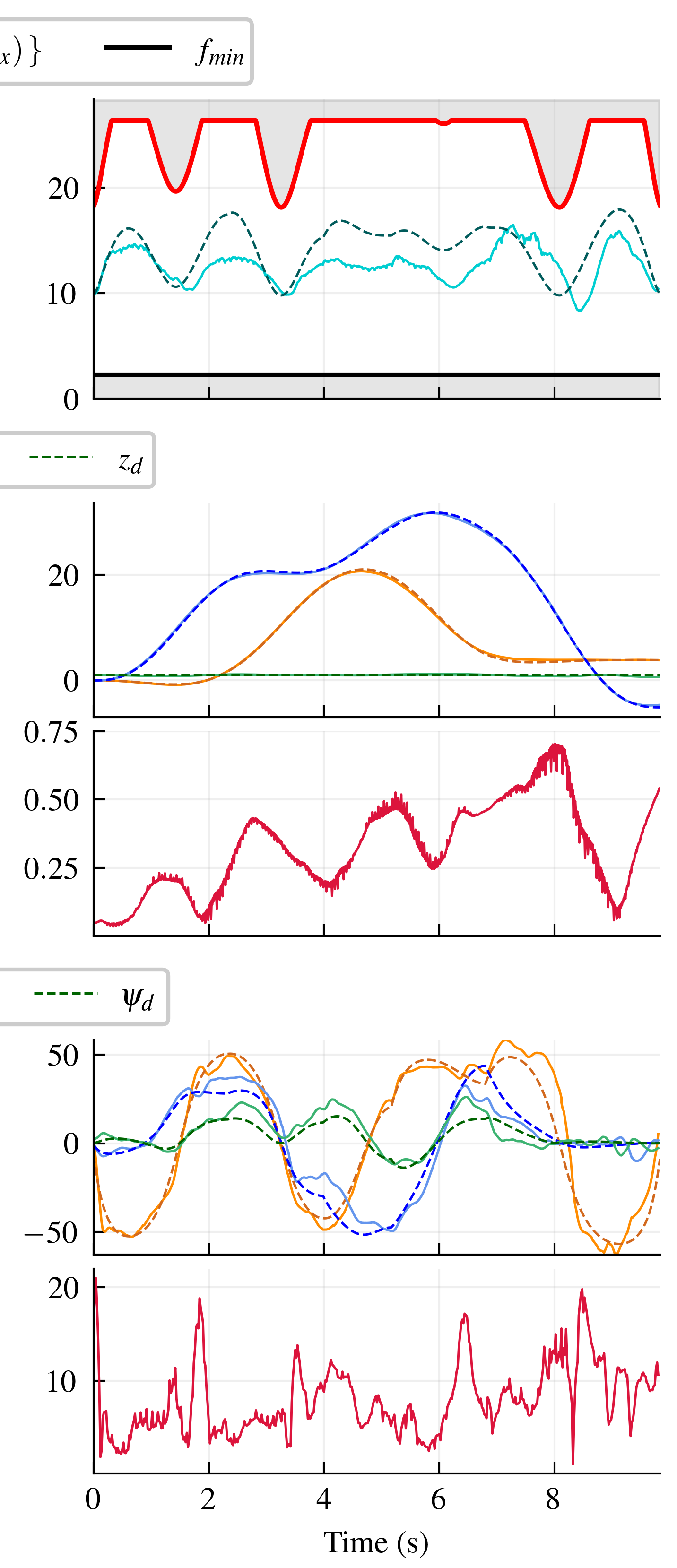}
    \caption{Experiment 6: State profiles, $v_{max} = $ 17.5~m/s.}
    \label{fig:hardware_exp6_plot}
\end{subfigure}
\caption{
 State profiles comparing the actual flown trajectory and STITCHER's planned trajectory for three of the high speed experiments: (a) Experiment 3 reaching 12.5 m/s, (b) Experiment 5  reaching 15 m/s, and (c) Experiment 6  reaching 17.5 m/s. Each figure has the following structure. Top: The actual (solid) and desired (dashed) mass-normalized thrust profiles remaining outside of invalid regions (grey). Middle: Comparison of the actual and desired position in each coordinate axis along with the total error (red). Bottom: The actual and desired roll $\phi$, pitch $\theta$, and yaw $\psi$ as well as the tilt error (red).}
\label{fig:hardware_experiments_state_profiles}
\vspace{-0.1in}
\end{figure*}

\begingroup

\renewcommand{\arraystretch}{1.25} 

\begin{table}[t]
\caption{Trajectory Tracking Error.}
\label{tab:tracking_error}
\centering
\begin{tabular}{c|c|c|c|c|c|c|c}
\hline
\multicolumn{2}{c|}{Metric} & \multicolumn{6}{c}{Experiment} \\
\multicolumn{2}{c|}{} & 1 & 2 & 3 & 4 & 5 & 6 \\
\hline
\multirow{2}{*}{\parbox{1.0cm}{\centering Path}}
& $v_{max}$ (m/s) & 10 & 10 & 12.5 & 15 & 15 & 17.5\\
& Length (m) & 55.8 & 110 & 163 & 116 & 199 & 96\\
\hline
\multirow{2}{*}{RMSE}
& Position (m) & 0.22 & 0.23 & 0.32 & 0.33 & 0.41 & 0.36\\
& Tilt (deg) & 4.81 & 6.12 & 6.90 & 7.00 & 7.30 & 8.80\\
\hline
\end{tabular}
\end{table}

\endgroup
\subsection{Flight Experiments} \label{sec:flight_exp}
Flight experiments were conducted in an underground parking structure to evaluate the dynamic feasibility of the planned trajectories generated by STITCHER. 
Trajectories were designed offline using point cloud maps generated by DLIO\footnote{The pre-maps were only used for planning, i.e., pre-maps were not used for localization. The full DLIO pipeline was used during each flight experiment.}.
\Cref{fig:rviz_hardware} shows the planned trajectories (red) generated by STITCHER along with the actual trajectories flown by the quadrotor (green) in six experiments of varying complexity. 
The trajectories were designed to aggressively avoid the structure columns of the parking garage, pushing the vehicle into regimes near its physical limits.
The commanded top speeds range between 10 m/s (36 km/h) to 17.5 m/s (63 km/h), and the path lengths range from 55 to 200 meters. 

\Cref{tab:tracking_error} reports the root mean square errors (RMSE) of position and tilt for each experiment in addition to the corresponding trajectory length and top speed. 
Across all experiments, the maximum position RMSE of 0.41 m occurs during our longest multi-loop trajectory experiment spanning 199 meters with a top speed of 15 m/s (\cref{fig:hardware_p7_3_vmax_15_loops}). The largest tilt RMSE of 8.8 degrees occurs during our fastest flight experiment with a top speed of 17.5 m/s (\cref{fig:hardware_p7_3_vmax_17p5}).
These flight experiments show that STITCHER generates trajectories that satisfy constraints and can be tracked with low error using a standard cascaded geometric controller at aggressive flight speeds.
In addition, while STITCHER may use any $p-$th order integrator with $p\geq2$ to model system dynamics, the results show that a triple integrator model is sufficient for quadrotor tracking control despite discontinuities in jerk at waypoints. 

\Cref{fig:hardware_experiments_state_profiles} compares the actual and desired profiles of the mass-normalized thrust (top), position (middle), and attitude (bottom) for three of the high-speed experiments: Experiment 3 reaching 12.5 m/s (\cref{fig:hardware_exp3_plot}), Experiment 5 reaching 15 m/s (\cref{fig:hardware_exp5_plot}), and Experiment 6 reaching 17.5 m/s (\cref{fig:hardware_exp6_plot}).
The quadrotor remains within the physical limits dictated by thrust magnitude and tilt constraints.
The position and attitude error profiles in all experiments further show that for the duration of the flight, the maximum deviations in position and tilt are less than 0.75 m and 25.5 degrees, respectively. 
This demonstrates that the controller maintains suitable trajectory tracking for speeds up to 17.5 m/s, despite some degradation in tracking performance at higher speeds.
With the increase in top speed from \cref{fig:hardware_exp3_plot} to \cref{fig:hardware_exp6_plot}, we begin to see a decrease in performance evident in the mass-normalized thrust tracking (top). 
This is indicative of aerodynamic effects (e.g., parasitic and profile drag) becoming significant at high speeds. 
Incorporating these effects into planning and control models is an important area of future work that could further improve high-speed tracking performance.

The speed profiles of the desired and actual trajectories for Experiment 3, with a commanded maximum speed of 12.5~m/s, Experiment 5, with a commanded maximum speed of 15~m/s, and Experiment 6, with a commanded maximum speed of 17.5~m/s, are shown in \cref{fig:hardware_vel_norm_plot}. 
Experiments 5 and 6 correspond to the longest and fastest hardware experiments, respectively. 
The quadrotor closely tracks the desired speed profiles, with root mean square tracking errors of 0.64 m/s, 0.72 m/s and 0.78 m/s for Experiments 3, 5, and 6, respectively.
Additionally, all experiments reach the prescribed maximum speed without violating the imposed velocity constraint.
These results demonstrate that STITCHER generates trajectories that safely and accurately express complex system-level constraints even at aggressive flight speeds.

\begin{figure}[t]
 \centering
 \begin{subfigure}{\linewidth}
\centering
\includegraphics[width=\linewidth]{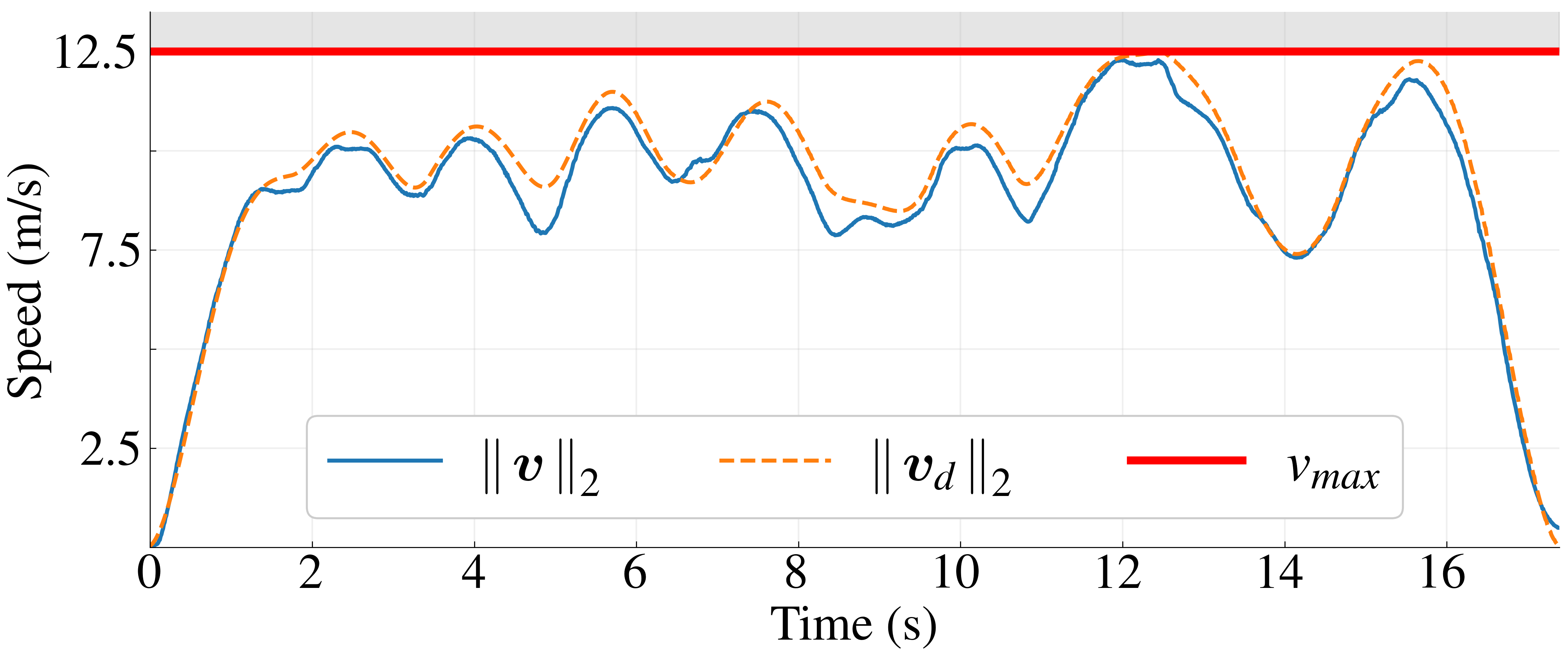}
 \caption{Experiment 3: Speed profile, $v_{max} =$ 12.5 m/s.}
\label{fig:hardware_vel_norm_12p5}
\vspace{6pt}
\end{subfigure} 
\\
\begin{subfigure}{\linewidth}
\centering
\includegraphics[width=\linewidth]{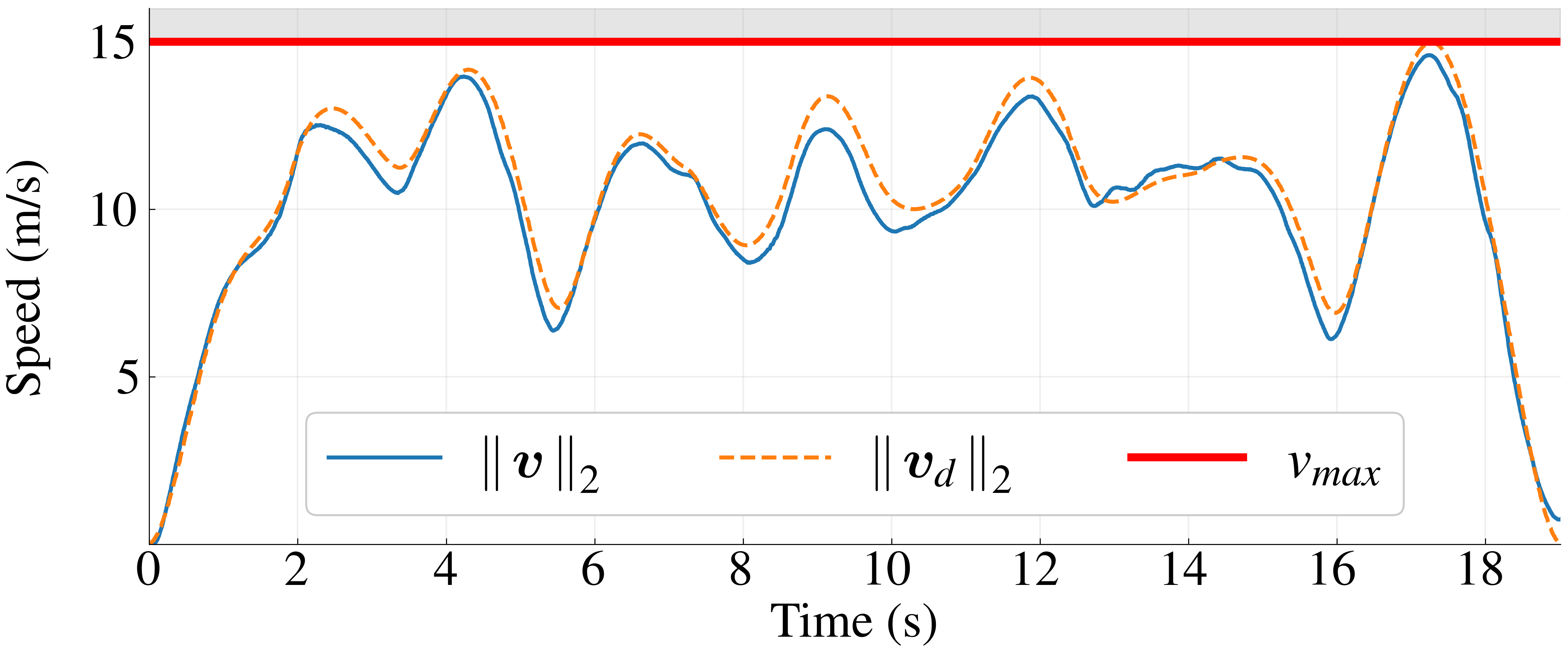}
 \caption{Experiment 5: Speed profile, $v_{max} =$ 15 m/s.}
\label{fig:hardware_vel_norm_15}
\vspace{6pt}
\end{subfigure} 
\\
 \begin{subfigure}{\linewidth}
\centering
\includegraphics[width=\linewidth]{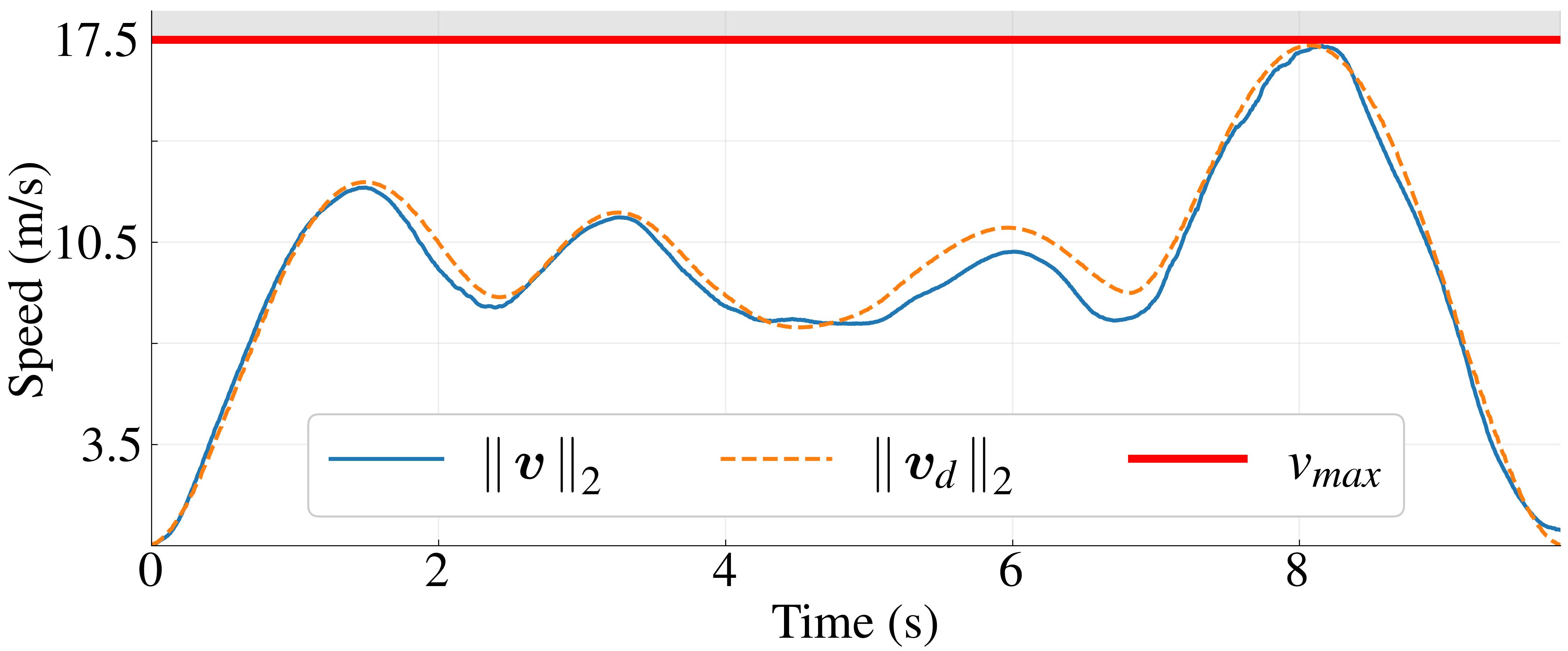}
 \caption{Experiment 6: Speed profile, $v_{max} =$ 17.5 m/s.}
\label{fig:hardware_vel_norm_17p5}
\vspace{6pt}
\end{subfigure} 
\caption{Speed profiles comparing the actual (blue) and desired (orange) trajectories reaching (a): 12.5 m/s, (b): 15 m/s, and (c): 17.5 m/s.}
\label{fig:hardware_vel_norm_plot}
\end{figure}

\subsection{Comparison with State-of-the-Art in Hardware}\label{sec:comp_soa}
Additional hardware experiments were performed to assess the impact of model fidelity on trajectory quality by comparing flight data of STITCHER and GCOPTER \cite{Wang22:Geometrically} trajectories. 
STITCHER uses a chain of integrators model for trajectory planning. Specifically in our experimentation, we use a triple-integrator model which yields discontinuities in jerk. 
From the hardware experiments performed in \cref{sec:flight_exp}, we saw that this model yields trackable high-speed trajectories for a standard cascaded controller. 
However, we recognize that planners such as GCOPTER may create smooth control-input profiles in jerk, resulting in smooth angular velocity commands. 
Therefore, here we examine how smoothness in jerk states practically affects tracking performance. 
We conducted three SOA comparison experiments comparing STITCHER to GCOPTER over the following trajectories: (1) a single-loop trajectory with a top speed of 5 m/s, (2) a multi-loop trajectory with a top speed of 10 m/s, and (3) a multi-loop trajectory with a top speed of 15 m/s. 
The experiments were performed by providing the same reference path through Parking Structure 7 to each respective planner. 
\Cref{fig:hardware_gcopter} depicts the trajectories planned by STITCHER and GCOPTER and the actual flight data for all experiments. 
In SOA Comparison Experiment 3, the GCOPTER trajectory could not be flown due to a collision violation caused by an invalid free space corridor (see \cref{fig:hardware_gcop_vmax_15_comp} left). 
This highlights a single point failure mode of optimization-based methods that require safe flight corridors in their formulation. 
In order to qualitatively assess trajectory quality, we compare the component-wise velocity profile of STITCHER and GCOPTER for SOA Comparison 2 (see \cref{fig:stticher_v_gcopter_vel}). 
The velocity profiles are very similar, with some deviation around 12 seconds. 

\begin{table}[t]
\caption{State-of-the-Art Comparison Trajectory Tracking Error.}
\vskip -0.1in  
\label{tab:SOA_traj_track}
\begin{center}
\begin{tabular}{c|c|c|c|c|c|c|c}
\hline
\multirow{2}{*}{\parbox{0.5cm}{\centering SOA Exp.}} & \multirow{2}{*}{\parbox{0.7cm}{\centering $v_{max}$ (m/s)}} & \multicolumn{2}{c|}{Path Len. (m)} & \multicolumn{2}{c|}{Pos. RMSE (m)} &\multicolumn{2}{c}{Tilt RMSE ($^\circ$)}\\
& & \cite{Wang22:Geometrically} &  Ours & \cite{Wang22:Geometrically} &  Ours &  \cite{Wang22:Geometrically} &  Ours  \\
\hline
1 & 5 & 48 & 51 & 0.13 & \textbf{0.12} & \textbf{2.65} & 4.28\\
2 & 10 & 140 & 153 & \textbf{0.29} & 0.30 & \textbf{4.92} & 5.88 \\
\hline
\end{tabular}
\end{center}
\vskip -0.1in  
\end{table}

\Cref{tab:SOA_traj_track} shows the tracking error in position and tilt for the successful SOA comparison experiments. GCOPTER is able to achieve a lower RMSE in tilt for both experiments. It should be noted that the maximum difference between the two planning methods is 1.6 degrees. STITCHER achieved a lower RMSE in position for Experiment 1, while GCOPTER achieved a lower value for Experiment 2. The difference in position RMSE for both experiments is 1 cm. From these results, we can conclude that for missions relying on accurate position tracking, STITCHER can reliably be used all while benefiting from fast computation speed. However, missions which require more precise tilt tracking may benefit from a smooth jerk profile.

In summary, these experiments show that STITCHER produces sufficiently smooth trajectories enabling low position tracking error. 
While there are added benefits in tilt tracking when using a higher-order planning model, we find the overall tracking error of the proposed STITCHER planner to be sufficient for many high-speed missions. 

\begin{figure*}[t]
\begin{subfigure}{0.49\linewidth}
\centering
 \includegraphics[width=\linewidth]{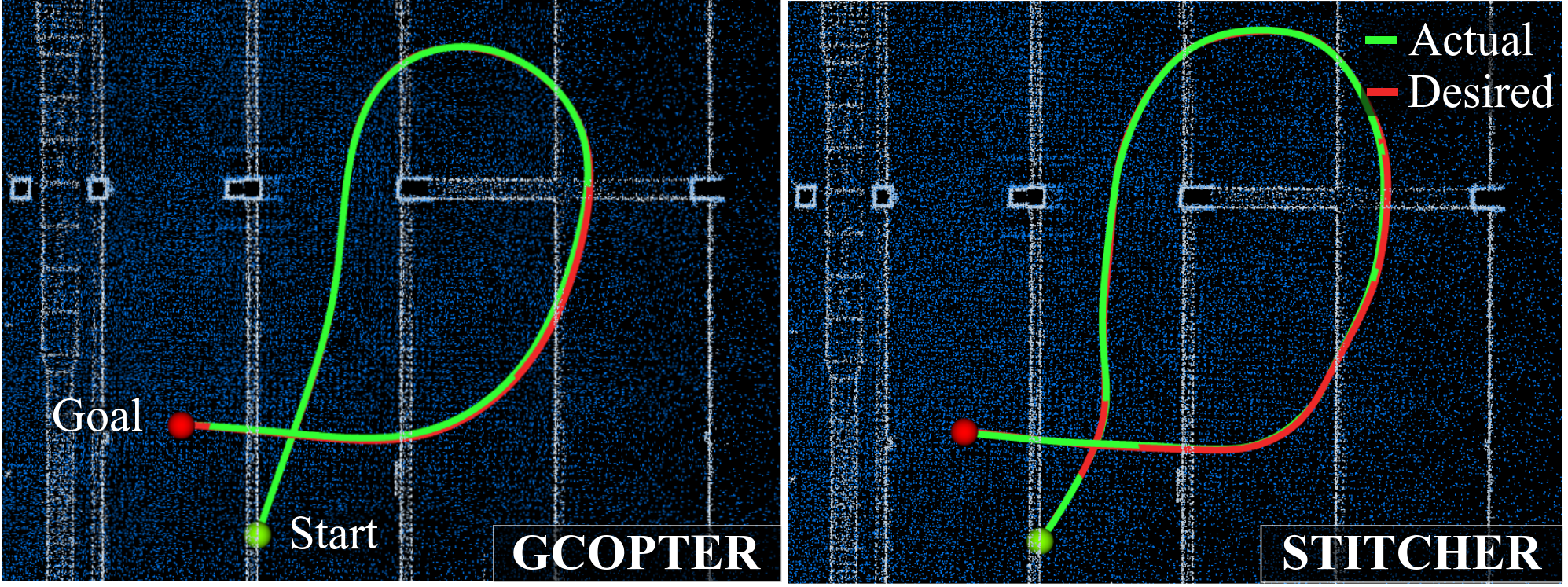}
 \caption{SOA Comparison 1: Single-loop trajectory,  $v_{max} = 5 $ m/s}
\label{fig:hardware_gcop_vmax_5}
\vspace{6pt}
\end{subfigure}
\hfill
\begin{subfigure}{0.49\linewidth}
\centering
 \includegraphics[width=\linewidth]{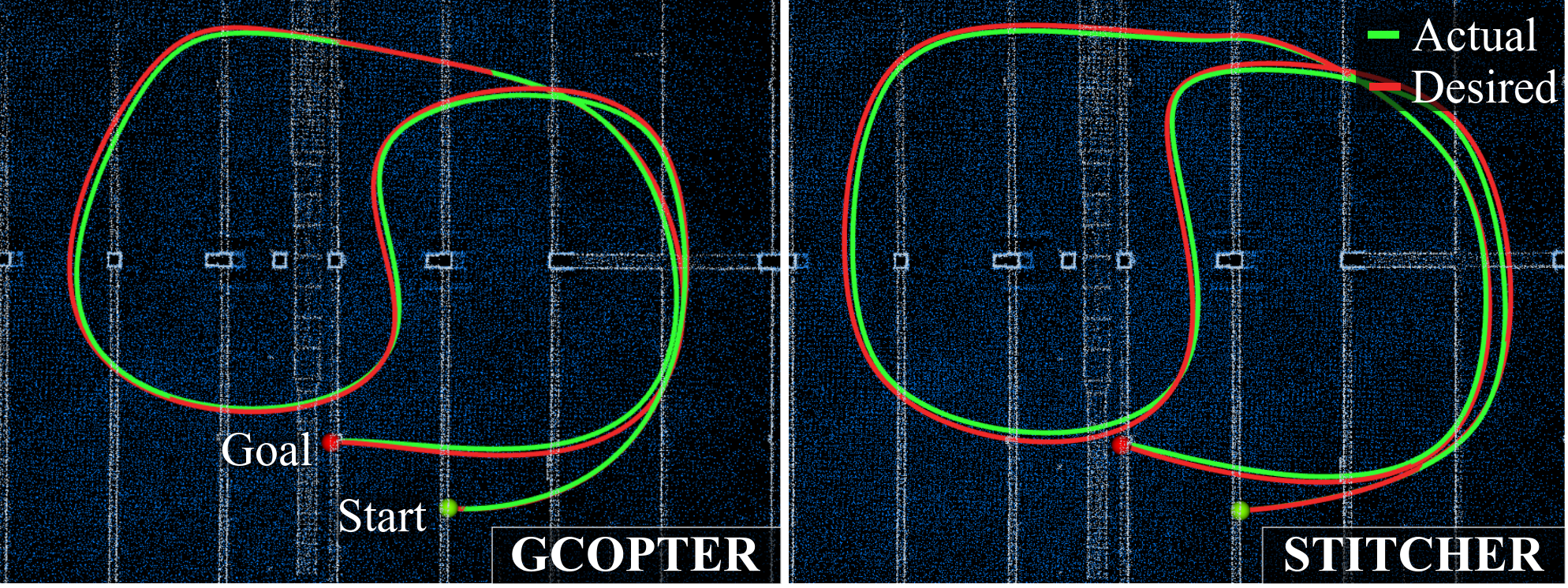}
 \caption{SOA Comparison 2:  Multi-loop trajectory,  $v_{max} = 10 $ m/s}
\label{fig:hardware_gcop_vmax_10_comp}
\vspace{6pt}
\end{subfigure}
\\
\centering
\begin{subfigure}{0.67\linewidth}
\centering
 \includegraphics[width=\linewidth]{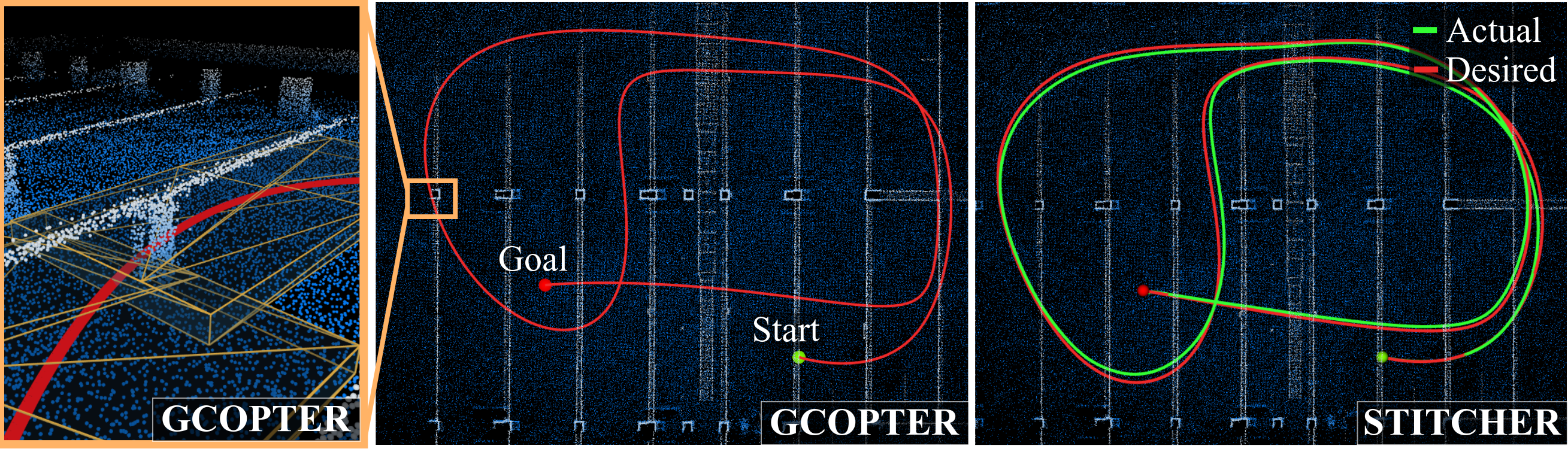}
 \caption{SOA Comparison 3:  Multi-loop trajectory,  $v_{max} = 15 $ m/s}
\label{fig:hardware_gcop_vmax_15_comp}
\end{subfigure}
\caption{Hardware experiments in 
Parking Structure 7 with three different trajectories: (a) Single-loop $v_{max} = 5 $ m/s, (b) Multi-loop $v_{max} = 10 $ m/s, and (c) Multi-loop $v_{max} = 15 $ m/s. The desired trajectory is shown in red and the actual trajectory flown by the quadcopter is shown in green. Two planner implementations are shown in each subfigure. Left: GCOPTER. Right: STITCHER.}
\label{fig:hardware_gcopter}
\vspace{-0.2in}
\end{figure*}

\begin{figure}[t]
\begin{subfigure}{\linewidth}
\centering
 \includegraphics[width=\linewidth]{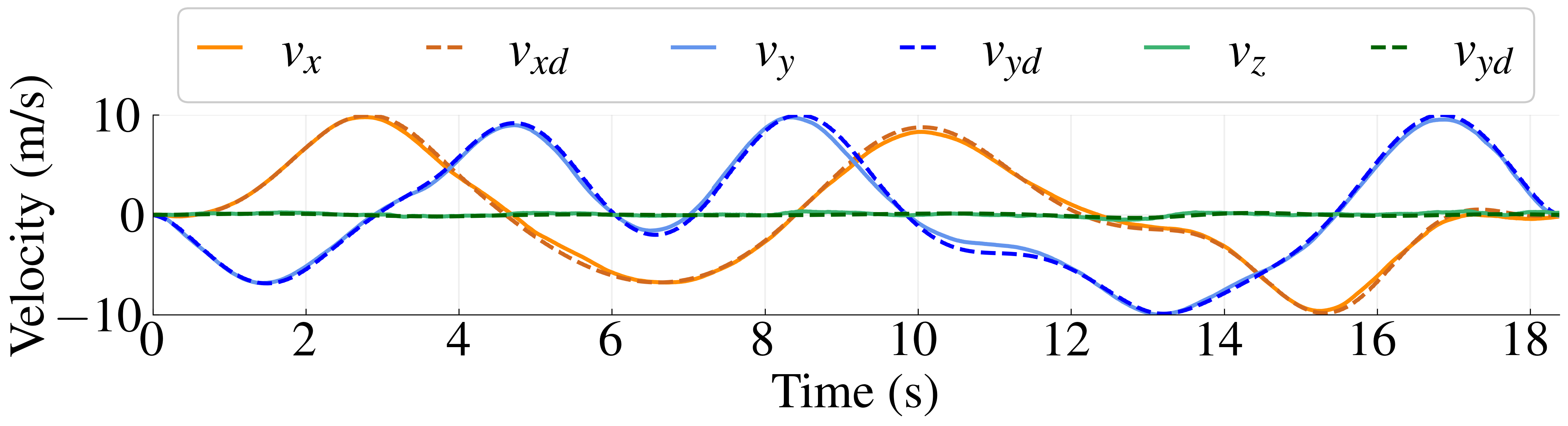}
 \caption{GCOPTER Velocity Profile}
\label{fig:gcopter_vel}
\vspace{6pt}
\end{subfigure}
\\
\begin{subfigure}{\linewidth}
\centering
 \includegraphics[width=\linewidth]{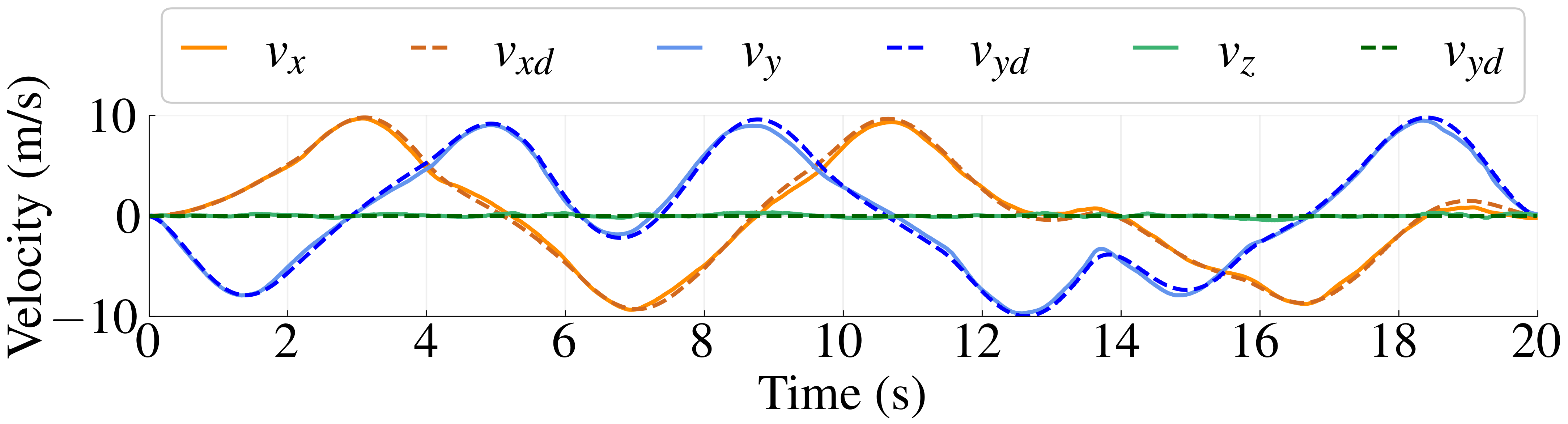}
 \caption{STITCHER Velocity Profile}
\label{fig:stitcher_Vel}
\end{subfigure}
\caption{SOA Comparison 2: Velocity tracking performance compared to state-of-the-art. (a): GCOPTER velocity profiles. (b): STITCHER velocity profiles. Dashed lines are desired profiles and solid lines are true profiles.}
\label{fig:stticher_v_gcopter_vel}
\end{figure}

\begingroup

\renewcommand{\arraystretch}{1.25} 

\section{CONCLUSIONS}
In this work, we presented STITCHER, a motion primitive search planning algorithm that utilizes a novel three-stage planning architecture to design constrained trajectories in real-time over long distances.
We proved the search graph is finite, and the proposed search heuristic is admissible, so STITCHER is guaranteed to i) have \textit{a priori} bounded time and memory complexity and ii) generate optimal trajectories with respect to the sampled set of states. 
Real-time search speeds were achieved through our novel heuristic crafting
technique, greedy graph pre-processing method, and non-conservative constraint and collision checking procedure.
Our extensive simulation study showed the trade-off in terms of path quality and computation time for different sampled velocity sets, the effectiveness of the proposed heuristic with varied edge costs and state constraints, a case study imposing individual thruster limits, and the average computation times of the components that make up STITCHER.
We also found that our greedy motion primitive graph pre-processing step has a negligible effect on solution cost compared to the observed computation speed up owing to the reduced graph size.
Importantly, STITCHER was shown to consistently generate trajectories faster than three state-of-the-art optimization-based planners while never violating constraints. 
Hardware experiments with flight speeds up to 63 km/h further proved that our planner is effective in generating trajectories suitable for position and attitude tracking control while remaining within set physical limits.

The current STITCHER implementation is designed for known environments, using a static \textit{kd}-tree representation.
Extending STITCHER to unknown environments would entail using an environment representation that supports online map updates, as incorporating new map data into the current framework would otherwise require rebuilding the \textit{kd}-tree.
Future work includes developing a receding horizon implementation for navigating through unknown environments, a free-waypoint formulation, the use of imitation learning to improve search efficiency, and learning motion primitives for more general optimal control problems.

\addtolength{\textheight}{0cm}   



\section*{ACKNOWLEDGMENT}
The authors would like to thank lab members Grace Kwak, Ryu Adams, Alan Song, James Row, Qiyuan Wu, Terrence Han, and Alexander Henderson for implementation and hardware support. 


\bibliographystyle{IEEEtran}
\bibliography{ref}
 
\vfill

\end{document}